\documentclass[sigconf]{acmart}
\setcopyright{none}
\renewcommand\footnotetextcopyrightpermission[1]{}
\AtBeginDocument{%
  }

\usepackage{appendix}

\makeatletter
\newif\if@restonecol
\makeatother

\usepackage[ruled,vlined]{algorithm2e}

\SetKwInOut{Input}{Input}
\SetKwInOut{Output}{Output}

\begin{document}

\title{DCFO: Density-Based Counterfactuals for Outliers - Additional Material}

\author{Tommaso Amico}
\authornotemark[1]
\email{tomam@cs.au.dk}
\affiliation{%
  \institution{Department of Computer Science}
  \city{Aarhus}
  \country{Denmark}
}
\author{Pernille Matthews}

\authornote{Both authors contributed equally to this research.}
\email{matthews@cs.au.dk}
\affiliation{%
  \institution{Department of Computer Science}
  \city{Aarhus}
  \country{Denmark}
}

\author{Lena Krieger}
\affiliation{%
  \institution{Forschungszentrum Jülich}
  \city{Jülich}
  \country{Germany}}
\email{l.krieger@fz-juelich.de}

\author{Arthur Zimek}
\affiliation{%
  \institution{Department of Computer Science and Mathematics}
  \city{Odense}
  \country{Denmark}
}

\author{Ira Assent}
\affiliation{%
 \institution{Department of Computer Science}
 \city{Aarhus}
 \country{Denmark}}






\begin{abstract}
    Outlier detection identifies data points that significantly deviate from the majority of the data distribution. Explaining outliers is crucial for understanding the underlying factors that contribute to their detection, validating their significance, and identifying potential biases or errors. Effective explanations provide actionable insights, facilitating preventive measures to avoid similar outliers in the future. Counterfactual explanations clarify why specific data points are classified as outliers by identifying minimal changes required to alter their prediction. Although valuable, most existing counterfactual explanation methods overlook the unique challenges posed by outlier detection, and fail to target classical, widely adopted outlier detection algorithms. Local Outlier Factor (LOF) is one the most popular unsupervised outlier detection methods, quantifying outlierness through relative local density. Despite LOF's widespread use across diverse applications, it lacks interpretability. To address this limitation, we introduce \emph{Density-based Counterfactuals for Outliers (DCFO)}, a novel method specifically designed to generate counterfactual explanations for LOF. DCFO partitions the data space into regions where LOF behaves smoothly, enabling efficient gradient-based optimisation. Extensive experimental validation on $50$ OpenML datasets demonstrates that DCFO consistently outperforms benchmarked competitors, offering superior proximity and validity of generated counterfactuals.
\end{abstract}




\maketitle

\section{Introduction}
    Outliers, or anomalies, are generally defined as a subset of observations that are inconsistent with the rest of the dataset~\cite{hodge2004survey, wang2019progress, zhou2025eace, angiulli2023counterfactuals, grubbs1969procedures, barnett1994outliers,ZimekF18}. Detecting these outliers is crucial across numerous applications, including identifying system faults, fraudulent activities, malicious transactions, and sensor malfunctions~\cite{hodge2004survey, panjei2022survey}.
    
    Explaining detected anomalies can be a key step in the detection pipeline. The explanation enables stakeholders to understand why the instance was flagged as an outlier, allowing them to take proper action~\cite{panjei2022survey}. In an application designed to detect malicious card transactions, a clear explanation can validate the outlier as a genuine attack or identify it as merely an unusual but non-malicious event. Moreover, as in other Explainable Artificial Intelligence (XAI) applications, interpreting the outputs of an outlier detection model builds trust in the model, and understanding the rationale behind its predictions enables actionability when the model does not behave as expected~\cite{Das2020OpportunitiesAC}. Despite the importance and impact of outlier detection models, their explanation is often overlooked~\cite{zhou2025eace}.

    Counterfactual explanations, a popular XAI technique, identify minimal changes required to adjust a model's prediction~\cite{romashov2022baycon, verma2020counterfactual}. In outlier detection contexts, counterfactuals highlight the specific features causing the outlier and suggest actionable modifications to convert anomalous instances into inliers. 
    
    For instance, consider an assembly line equipped with numerous temperature and pressure sensors that typically report values within a consistent range. If one sensor detects a temperature outside this range, a Local Outlier Factor (LOF)~\cite{breunig2000lof} model would flag this as anomalous. Stakeholders would then question if this spike indicates a genuine overheating issue or simply a sensor malfunction. A counterfactual explanation would clarify the precise adjustment needed to restore normal readings, guiding practical interventions. A raw LOF score alone would not suffice, as the method provides no attribution of which features are responsible for the outlier and offers no guidance on which changes would be required to bring the measurement back into the normal operating region.
    
    However, standard outlier detection techniques, such as LOF, present significant challenges for traditional counterfactual methods. LOF calculates outlierness based on local density comparisons among neighbouring points. Consequently, small changes in an observation can abruptly change its nearest neighbours, causing discontinuous jumps in outlier scores. Such discontinuities render gradient-based optimisation methods ineffective since they depend on smooth prediction transitions.

    Alternative approaches like Bayesian optimisation handle discontinuities by exploring many possibilities without gradients, but their computational cost becomes prohibitive in high-dimensional settings~\cite{romashov2022baycon}.
    Counterfactual methods have traditionally targeted classification models. Although outlier detection can be framed this way, it is highly imbalanced, making counterfactual search difficult~\cite{zhou2025eace}. Outliers are a small minority, often far from inliers and even from each other~\cite{zhou2025eace}.
    

    In this work, we propose \emph{Density-based Counterfactuals for Outliers (DCFO)}, a novel counterfactual method designed explicitly for explaining LOF outputs~\cite{breunig2000lof}. DCFO exploits the inherent structure of LOF by partitioning the data space into meaningful regions, enabling efficient gradient-based optimisation despite the underlying non-continuous nature of LOF scores. Our method also supports handling non-actionable features, produces diverse counterfactuals, and is able to adjusts counterfactuals' plausibility.

    Our main contributions include:
    \begin{itemize}
      \item The first Local Outlier Factor counterfactual explanation method, DCFO, which partitions the data space into regions, and subsequently performs gradient-based optimisation to produce counterfactuals. 
 
      \item DCFO's algorithmic modifications to handle non-actionable features, generate multiple diverse counterfactuals, and take plausibility into account.
      
      \item Empirical benchmarks using real-world datasets, showing how DCFO's counterfactuals perform better than competitors in terms of proximity, validity and diversity.
    \end{itemize}

    In the remainder of the paper, we discuss related works in Section~\ref{sec:related}. DCFO is thoroughly introduced in Section~\ref{sec:method}. Experimental evaluation is carried out in Section~\ref{sec:experiments}, while conclusions and discussion on future works are in Section~\ref{sec:conclusion}.

\section{Related Works}\label{sec:related}
We first introduce outlier detection methods in Section~\ref{subsec:OD}, and then examine explainable outlier detection approaches in Section~\ref{subsec:XOD}.

\subsection{Outlier Detection Methods}\label{subsec:OD}
Outlier detection has a history of unsupervised methods identifying data points which deviate from the norm. Early approaches include distance-based methods, which flag points whose distance to their $k$-nearest neighbours exceeds a predefined threshold, and density-based methods, which consider the sparsity of local neighbourhoods to identify anomalies. A prominent example of the latter is Local Outlier Factor (LOF), which computes an ``outlier score'' comparing the local density of a point to that of its neighbours, marking points with significantly lower density as outliers~\cite{breunig2000lof}. 

Given that LOF introduces a degree of outlierness rather than a binary inlier/outlier categorisation, variants and related methods to LOF include Simplified LOF (SLOF)~\cite{schubert14local}, Influenced Outlierness (INFLO)~\cite{jin06ranking}, and Dimensionality-Aware Outlier Detection (DAO)~\cite{Anderberg0CHMRZ24}, which modify how local density is measured and compared. SLOF modifies the computation by replacing the reachability distances with the inverse $k$-nearest neighbour distance; this avoids calculating the reachability computation step, to reduce the time complexity. INFLO considers reverse nearest neighbours, which are the points that have the query point as one of their nearest neighbours. The reverse nearest neighbours are incorporated into the density estimation, providing more reliable outlier scores in sparse regions where traditional LOF may underestimate the density. DAO estimates local intrinsic dimensionality and uses a density-estimation adjusted to the estimated local dimensionality.

Other popular outlier detection methods include the $k$-nearest neighbours (KNN) approach~\cite{ramaswamy2000efficient} and a method based on variance of angles (ABOD)~\cite{KriegelSZ08}. KNN assigns an outlier score to each point equal to its distance to the $k-$th nearest neighbour, effectively capturing global deviations. The ABOD outlier score takes into account how much the angles needed to see pairs of other points vary -- anomalous points exhibit smaller angle variance among pairs of other points than inliers.

Clustering-based methods~\cite{VincesSZC25} include $k-$means$--$~\cite{chawla13kmeans}, which integrates clustering and outlier detection by iteratively removing points with the greatest distances from their nearest cluster centroids before updating centroid positions, thereby identifying points that lie far from all clusters as outliers.
The hierarchical density-based method GLOSH~\cite{campello15hierarchical} extends LOF's ideas by leveraging hierarchical density estimations (otherwise used for hierarchical clustering) to detect outliers across multiple density levels.

\subsection{Explainable Outlier Detection (XOD)}\label{subsec:XOD}
Explainable Outlier Detection (XOD) achieves two objectives. XOD first identifies outlying instances and then provides explanations for why an instance is considered an outlier. One can consider two categories when considering XOD methods: (i) shallow and symbolic explanations, which explain outliers through features or rules, and (ii) counterfactual explanations, which reveal how an outlying instance can be modified to become an inlier~\cite{panjei2022survey,li2023survey,sejr2021explainable}.

\paragraph{Shallow XOD Methods} 
Early and popular XOD methods focus on describing outliers through feature attributions, summaries and logical rules. One of the earliest works is from Knorr et al. with their notion of intensional knowledge~\cite{knorr99finding}. Rather than simply finding an outlier, their method identifies a minimal set of attribute constraints isolating the instance from its nearest neighbours.

Consider an example for credit-card fraud detectors. A specific outlying transaction is captured by the explanation: 
    $$
    \mbox{\{\texttt{Amount} > 1{,}000, \texttt{HourOfDay} < 6\}}.
    $$
    
The intensional-knowledge method would mine two conditions: ``High amount'' and ``Odd time'', together they separate the transaction from its $10$ nearest non-fraud examples in the feature space. When analysts read such an explanation, they know immediately: \textbf{(i)} large late-night purchases are an issue, and \textbf{(ii)} no other features are needed to explain this outlier.

The approach has three key benefits, which are consistently observed in the XOD literature. Conciseness, as the explanation uses as few features as possible to explain the outlier, interpretability, as the explanation is immediately interpretable by non-technical users. Lastly, validation is possible; it is fast to check the data distribution, and analysts can validate whether the flagged point truly deviates from normal patterns.

Anguilli et al.~\cite{angiulli09detecting} define outlying properties by quantifying how extreme each attribute and outlier values are relative to the overall distribution. Müller et al.~\cite{muller12outrules} extends this notion to OutRules, generating logical rules across multiple contexts; characterising outlying behaviour within each context. While these methods provide symbolic explanations, they provide only static descriptions of \emph{``why''} an outlier is an outlier, but do not include actionable changes, i.e., \emph{``how''} to fix the outliers. 

\paragraph{Counterfactual Explanations}
A counterfactual explanation for an outlier would answer ``How could this outlier be modified to appear normal?''. Answering the important ``how'' question is crucial for outlier detection as it provides insight into why an instance is an outlier and how it can be ``fixed''.

Some methods use counterfactuals for anomaly detection~\cite{liguori2024robust, fontanella2024diffusion}, e.g., generating healthy brains as counterfactuals to detect anomalies in brain images~\cite{fontanella2024diffusion}, while most focus on explaining identified outliers.

Several methods for explaining outliers focus on data types like time series or categorical data. Counterfactual ensemble explanations~\cite{sulem2022diverse} explain time series identified as anomalous by a detection method. The approach is limited to differentiable anomaly detection methods.
Trifunov et al.~\cite{trifunov2021anomaly} propose an attribution scheme based on the Maximally Divergent Interval algorithm. The idea consists in changing a subset of features to align with the respective feature distribution of the time series and recalculate how the anomaly changed.
Context preserving Algorithmic Recourse for Anomalies in Tabular data (CARAT)~\cite{datta2022framing} is a transformer-based encoder-decoder model, solely for categorical data.

More recently, methods solving counterfactual generation as an optimization problem based on newly defined loss functions have been proposed.
Explain Anomaly via Counterfactual Explanations (EACE)~\cite{zhou2025eace} is an approach based on optimizing a loss function which includes a classification change term, a density term, and a boundary loss component. The density is estimated through LOF. To solve the optimization problem, they propose a genetic algorithm and thus, EACE can work also with non differentiable models like LOF. The proposed genetic algorithm performs suboptimally in datasets where outlying instances are not extremely distant from other points~\cite{zhou2025eace}. 
Angiulli et al.~\cite{angiulli2023counterfactuals} propose Masking Models for Outlier Explanation (M$^2$OE): a deep-neural network guided by an ad-hoc loss function.
The generated explanations consist of a set of features and a respective set of changes. The loss function incorporates the isolation of the input outlier before and after applying the generated mask within the subspace defined by the features choice.

Another approach to counterfactual explanations is solving a feasible path finding problem, i.e., constructing a path from the data point to be explained towards the counterfactual.
Two methods employ LOF to solve this problem. Yamao et al.~\cite{yamao2024distribution} sequentially change the input feature vector until it obtains the desired classification result, employing a cost function based on LOF to assess feasibility and actionability of the path. Their proposal struggles with validity and does not account for non-actionable features.

DACE~\cite{kanamori2020dace} follows a similar approach: to obtain a feasible perturbed vector, they introduce a cost function considering feature correlation and data distribution solved with mixed-integer linear optimization. To account for the data distribution they employ LOF. DACE focuses on additive classifiers only.

Baycon~\cite{romashov2022baycon} is a counterfactual method that uses probabilistic feature sampling and Bayesian optimisation. Its objective balances similarity in both the feature and output space while minimising the number of modified features. Although not solely designed specifically for outliers, Baycon is widely used in our benchmarks, and its Bayesian formulation makes it well-suited for non-differential models like LOF. 

Closely related to counterfactuals is the idea of outlier correction, i.e., transforming outliers into inliers. Ji et al.~\cite{ji2024ar} use generative models to achieve this by producing diffusion-based counterfactuals, with a primary focus on diffusion techniques.
\section{Methods}\label{sec:method}
    The goal of DCFO is to generate counterfactual explanations for outliers identified by Local Outlier Factor (LOF). In this section, we first introduce LOF. Next, we present our proposed DCFO method and outline the steps for generating multiple diverse counterfactuals. Finally, we demonstrate how DCFO addresses non-actionable features and the plausibility of the proposed counterfactuals.


\subsection{Local Outlier Factor}\label{subsec:lof}
    Local Outlier Factor, first proposed by Breunig et al.~\cite{breunig2000lof}, scores the outlierness of points based on their relative density, i.e., comparing their density with that of their neighbours. A point, whose density is comparable with that of its neighbours, has an LOF score of approximately $1$. The lower the density of the point compared to its neighbours, the higher the outlierness. The original paper suggests a threshold of $1.5$ as good practice for separating inliers and outliers, i.e., a point is identified as an outlier if the average density of its neighbours is one and a half times higher than the density of the point itself. Another common practice for setting the threshold is to take a small quantile of the LOF scores distribution, e.g., $0.1$.
    
    We assume $\mathcal{D}$ to be a dataset comprising $n$ data points $\mathbf{p_1}, \dots, \mathbf{p_n}$. The notation $\mathbf{x}$ represents a location (or point) in the feature space $\mathcal{X}$. The function $d$ denotes the distance measure used to compute how far apart two points are in the feature space.
    
    To evaluate the $\operatorname{LOF}$ score of a point $\mathbf{p}$, we need to first compute its $k$-distance, i.e., the distance of its $k$-th nearest neighbour. $k$ is the singular parameter required by $\operatorname{LOF}$, and defines the notion of density by defining the neighbourhood of a point. We denote the set of $k$-nearest neighbours of a point $\mathbf{p_i}$ as $knn(\mathbf{p_i})$.

    To reduce statistical fluctuations among nearby points, the LOF algorithm employs the \emph{reachability distance}, which, while not being a true metric, is designed to stabilise density estimates. The reachability distance $rd_k(\mathbf{p_i}, \mathbf{p_j})$ between points $\mathbf{p_i}$ and $\mathbf{p_j}$ is defined as:
    
    \begin{equation}
        rd_k(\mathbf{p_i}, \mathbf{p_j}) = \max \left \{ k\text{-distance}(\mathbf{p_j}), d(\mathbf{p_i}, \mathbf{p_j}) \right \}.
    \end{equation}
    
    Specifically, when a point $\mathbf{p_j}$ lies closer than the $k$-distance of $\mathbf{p_i}$, the reachability distance $rd_k(\mathbf{p_i}, \mathbf{p_j})$ is set to $k\text{-distance}(\mathbf{p_j})$. This adjustment mitigates noise from very close points and yields more robust outlier scores~\cite{breunig2000lof}.
    
    LOF's notion of local density is represented by the local reachability density ($lrd_k$) of a point $\mathbf{p_i}$:
    \begin{equation}\label{eq:lrd}
            lrd_k(\mathbf{p_i}) = \dfrac{k}{
            \sum\limits_{\mathbf{p_j}\in knn(\mathbf{p_i})} rd_k(\mathbf{p_i}, \mathbf{p_j})}
             \quad ,
        \end{equation}
    which is the inverse of the average reachability distance from $\mathbf{p_i}$ to its neighbours.
        The $\operatorname{LOF}$ score of $\mathbf{p_i}$ is defined as:
    \begin{equation}\label{eq:LOF}
            \operatorname{LOF(\mathbf{p_i})} = \dfrac{1}{k \cdot lrd_k(\mathbf{p_i})} \cdot \sum\limits_{\mathbf{p_j} \in knn(\mathbf{p_i})}lrd_k(\mathbf{p_j}). 
        \end{equation}
    Hence, in Eq.~\ref{eq:LOF}, $\operatorname{LOF}$ compares the local density of a point $\mathbf{p_i}$, captured by its local reachability density ($lrd$), with the average $lrd$ of its $k$ nearest neighbours. 

    Formally, any $\mathbf{p_i}$ with $\operatorname{LOF}(\mathbf{p_i}) > t$ is labelled as an outlier.

    \subsection{Density-based Counterfactuals for Outliers}\label{subsec:dcfo}
        Density-based Counterfactuals for Outliers (DCFO) generates counterfactual proposals for data points flagged as outliers by $\operatorname{LOF}$. Given an outlier point $\mathbf{p_i}$, our goal is to find a nearby point $\mathbf{x}$ in the feature space such that, if $\mathbf{p_i}$ were hypothetically moved to $\mathbf{x}$, its LOF score would fall below the outlier threshold $t$. Thus, in the following, $\operatorname{LOF}(\mathbf{x})$ refers to the LOF score that would be assigned to $\mathbf{p}_i$ if it were hypothetically located at position $\mathbf{x}$.
        
        To evaluate $\operatorname{LOF}(\mathbf{x})$ in this context, we compute the score at location $\mathbf{x}$, treating $\mathbf{x}$ as a new data point and excluding the original $\mathbf{p_i}$ from all $k$-NN computations. This ensures that $\mathbf{x}$ is judged solely based on its local environment in feature space.
    
        Formally, denoting the resulting counterfactual as $\mathit{cf}(\mathbf{p_i})$, we solve the following constrained optimisation problem:
        \begin{equation}\label{eq:optimisationProblem}
            \mathit{cf}(\mathbf{p_i}) = \min_\mathbf{x \in \mathcal{X}} d(\mathbf{x}, \mathbf{p_i}) \quad \big | \quad  \operatorname{LOF}(\mathbf{x}) \le t \quad
        \end{equation}
    
        \noindent In Eq.~\ref{eq:optimisationProblem}, the optimisation problem seeks the smallest possible change, measured with distance $d$, that transforms an outlier into an inlier under the LOF model. As distance metric, we use Euclidean, as it is widely used and the default in LOF. However, a key challenge is that the constraint ($\operatorname{LOF}(\mathbf{x}) \le t$) is not differentiable, preventing the use of standard gradient-based optimisers.

        To address this, we partition the data space into regions where the LOF score behaves continuously and is differentiable. Within each region, we can apply gradient-based optimisation techniques to search for valid counterfactuals efficiently.

        To compute the $\operatorname{LOF}$ score of a candidate point $\mathbf{x}$, we need more than just the location of $\mathbf{x}$ itself. According to Eq.~\ref{eq:lrd} and Eq.~\ref{eq:LOF}, the LOF score depends on the local reachability densities of $\mathbf{x}$ and of its $k$ nearest neighbours. In turn, these densities require access to each neighbour's neighbourhood. Therefore, to evaluate $\operatorname{LOF}(\mathbf{x})$, we must identify \textbf{(i)} the $k$ nearest neighbours of $\mathbf{x}$, denoted as $knn(\mathbf{x})$, and \textbf{(ii)} for each neighbour $\mathbf{p_j} \in knn(\mathbf{x})$, their respective $k$ nearest neighbours $knn(\mathbf{p_j})$.

        We use this neighbourhood structure to define a unique \emph{key} that characterises a region of the space. Given point $\mathbf{x}$, we define its key:
        \[
        \mathcal{K}(\mathbf{x}) \equiv \left( knn(\mathbf{x}),\ knn(\mathbf{p_j})\ \forall\ \mathbf{p_j} \in knn(\mathbf{x}) \right).
        \]
        This key encodes all data points whose positions influence the LOF score of $\mathbf{x}$. We then define a region $R_K$ as the set of all points sharing the same key $K$:
        \begin{equation}\label{eq:regions}
                R_K \equiv \left\{ \mathbf{x} \in \mathcal{X} \,\big|\, \mathcal{K}(\mathbf{x}) = K \right\}.
        \end{equation}
    
        These regions $R_K$ form a partition of the input space, such that within each region, the set of   points relevant to LOF computations remains fixed. This property ensures, as will be proved in the following, that the LOF score varies smoothly within a region, making it suitable for gradient-based optimisation. As stated above, $\operatorname{LOF(\mathbf{x})}$, and consequently the space partition, depends on the point $\mathbf{p}_i$ being moved.

        Note that the symbol $K$ used to denote a region's key is unrelated to the lowercase $k$ used in $\operatorname{LOF}$ to indicate the number of neighbours. Also, if the key were defined using only $knn(\mathbf{x})$ (and not the neighbours-of-neighbours), the resulting partition would correspond to a $k$-th order Voronoi diagram, i.e., a decomposition of the space based on each point's $k$ nearest neighbours in the dataset.

        $\operatorname{DCFO}$ does not require computing the whole partition of the feature space in advance. Instead, it only estimates the key $\mathcal{K}(\mathbf{x})$ for selected points encountered during optimisation. This design choice is crucial for scalability: computing the complete $k$-th order Voronoi diagram is known to be computationally expensive, especially in high dimensions~\cite{aurenhammer1996voronoi}. The regions defined by Eq.~\ref{eq:regions} are even more complex than standard Voronoi cells, since they depend not only on $knn(\mathbf{x})$ but also on the neighbours of each neighbour. Constructing the entire space partition would, therefore, be inefficient and unnecessary.
    
        \begin{figure}[tbp]
            \centering
            \includegraphics[width=1\linewidth]{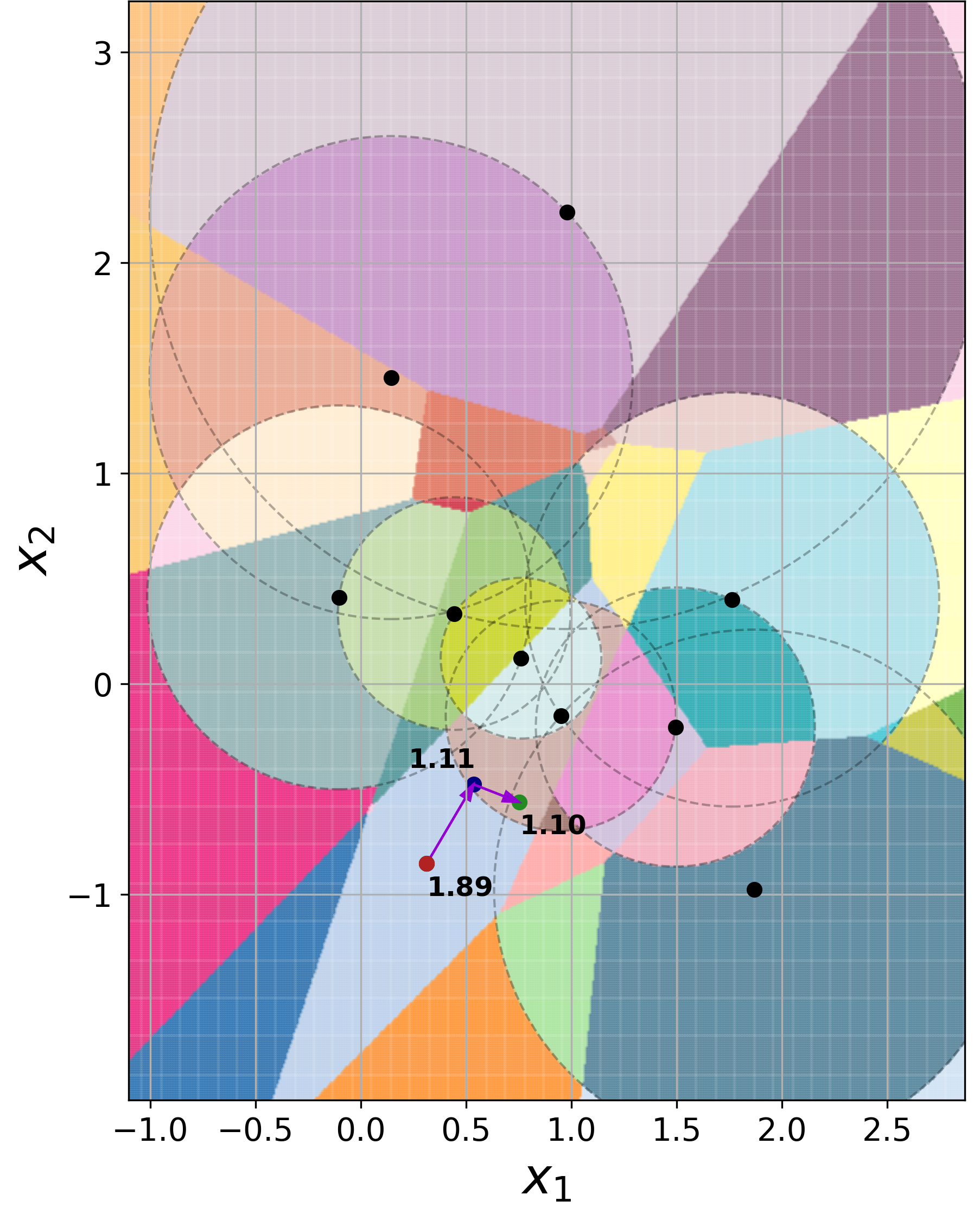}
            \caption{Visual example of DCFO's region-based space partitioning for a synthetic $2$D dataset. Each black point represents a data instance in $\mathcal{D}$, and the red point $\mathbf{p_i}$ is the outlier to be explained (threshold = $1.10$). Dashed circles denote $k$-distance neighbourhoods (with $k=2$). Background colours uniquely indicate the LOF regions. The blue point is the result of the first constrained optimisation, lying outside the initial region and still above the threshold ($\operatorname{LOF} = 1.11$). A second optimisation starting from that point produces the green counterfactual, which remains within a single region and satisfies the threshold ($\operatorname{LOF} = 1.10$). Arrows indicate the optimisation path.}
            \label{fig:spacePartition}
            \Description{}
        \end{figure}
    
        Figure~\ref{fig:spacePartition} illustrates the partitioning induced by keys $\mathcal{K}(\mathbf{x})$ on a $2$D synthetic dataset. The dataset $\mathcal{D}$ is sampled from a standard bivariate normal distribution. The red point $\mathbf{p_i}$ denotes the outlier for which a counterfactual is to be computed, while the black points represent other points in the dataset. The neighbourhood size is set to $k = 2$. The dashed circles illustrate the $k$-distance neighbourhood of each point $\mathbf{p_j} \in \mathcal{D}$. Each coloured region corresponds to a unique key $K$, i.e., all points in a region share the same nearest neighbours and second-order neighbours.

        Critically, within each region $R_K$, assuming a well-behaved distance function such as the one used for LOF, i.e., the Euclidean distance, the LOF score becomes a differentiable function of $\mathbf{x}$. More precisely, we denote $\operatorname{LOF}_K(\mathbf{x})$ as the LOF score computed using the fixed neighbourhood structure encoded by key $K$. This function is differentiable almost everywhere in $\mathcal{X}$, which allows us to perform efficient gradient-based optimisation. We formalise this property in proposition~\ref{prop:diff} below, the full proof is in Appendix~\ref{sec:proof}.

        \begin{proposition}\label{prop:diff}
            Let $d$ be a well-behaved distance function. Then, $\operatorname{LOF}_K(\mathbf{x})$ is continuous and differentiable almost everywhere. Specifically, if $d$ is twice continuously differentiable, then $\operatorname{LOF}_K(\mathbf{x})$ is twice continuously differentiable almost everywhere.
        \end{proposition}

        The algorithmic definition of $\operatorname{DCFO}$ is as follows. The method   takes as inputs: \textbf{(i)} the outlier $\mathbf{p}_i$ (i.e., a point in $\mathcal{D}$ with $\operatorname{LOF}(\mathbf{p}_i) > t$), \textbf{(ii)} the LOF parameter $k$, \textbf{(iii)} the threshold $t$ used to classify outliers, and \textbf{(iv)} the distance function $d$.

    The first step of the algorithm is to compute the key $K = \mathcal{K}  (\mathbf{p}_i)$, which defines the region $R_K$. We then solve the constrained optimisation problem in Eq.~\ref{eq:optimisationProblem}, replacing the original constraint $\operatorname{LOF}(\mathbf{x}) \le t$ with the relaxed, region-specific constraint $\operatorname{LOF}_K(\mathbf{x}) \le t$.

    Because $\operatorname{LOF}_K(\mathbf{x})$ is differentiable almost everywhere (Proposition~\ref{prop:diff}), we apply a gradient-based optimisation method. Specifically, we use Sequential Least Squares Programming (SLSQP)~\cite{kraft1988software, nocedal1999numerical}, an efficient constrained optimiser that handles inequality constraints.
    
    During the optimisation, if SLSQP explores a point $\mathbf{x}$ that lies outside the original region, we store it in a First-In-First-Out (FIFO) exploration queue for later evaluation. Once SLSQP converges to a solution $\mathbf{x}_f$, we check whether the final point remains within the original region. We compute the key $K' = \mathcal{K}(\mathbf{x}_f)$ and distinguish two cases. \emph{Case one:} If $K' = K$, then $\operatorname{LOF}_K(\mathbf{x}_f) = \operatorname{LOF}(\mathbf{x}_f)$, so the constraint is valid, and we return $\mathbf{x}_f$ as the counterfactual $\mathit{cf}(\mathbf{p}_i)$. \emph{Case two:} If $K' \neq K$, the constraint $\operatorname{LOF}_K(\mathbf{x}_f) \le t$ no longer guarantees that $\operatorname{LOF}(\mathbf{x}_f) \le t$, as the neighbourhood has changed. In this case, we recursively restart the algorithm using $\mathbf{x}_f$ as the new starting point, now operating in region $R_{K'}$.

    To avoid redundant computation, if $K'$ has already been the start of an optimisation, we dequeue the next unexplored point and attempt optimisation from there. The recursion terminates in two scenarios, \textbf{(i)} the optimiser finds a valid counterfactual in the current region,
    \textbf{(ii)} all candidates in the exploration queue have been exhausted without success. The step by step algorithm is outlined in Algorithm~\ref{alg:dcfo}.

    \begin{algorithm}[tb]
        \SetAlgoLined
        \Input{Outlier $\mathbf{p}_i$, distance $d$, $\operatorname{LOF}$ parameter $k$, dataset $\mathcal{D}$, threshold $t$.}
        \Output{Counterfactual, $\mathit{cf}(\mathbf{p}_i)$}
        \vspace{1mm}
        Initialise empty starting regions' queue, StartList $\leftarrow \left [ \ \right ]$\;
        \vspace{1mm}
        Initialise empty exploration queue, ExpList $\leftarrow \operatorname{FIFO} \left [ \ \right ]$ \;
        \vspace{1mm}
        Compute key $K=\mathcal{K}(\mathbf{p}_i)$ \;
        \vspace{1mm}
        StartList.append($K$) \;
        \vspace{1mm}
        Find $\mathbf{x}_f$ solving $\min\limits_{\mathbf{x}\in \mathcal{X}} d(\mathbf{x}, \mathbf{p}_i) \ \big | \ \operatorname{LOF_K}(\mathbf{x}) \leq t$, using SLSQP\;
        \vspace{1mm}
        Compute key $K'=\mathcal{K}(\mathbf{x}_f)$ \;
        \vspace{1mm}
        \If{$K' = K$}{
            \Return{$\mathbf{x}_f$} \;
            }
        \ElseIf{$K' \neq K$}
        {
        \If{$K' \notin \operatorname{StartList}$}{
        StartList.append($K'$) \;
        \Return{
        $\operatorname{DCFO}(\mathbf{x}_f)$
        \;}
        }
        \ElseIf{$K' \in $ StartList}{
        $\mathbf{x_i} \leftarrow \operatorname{ExpList}. \operatorname{pop}()$ \;
        \Return{
        $\operatorname{DCFO}(\mathbf{x}_i)$
        \;}
        }
    }
    \caption{$\operatorname{DCFO}(\mathbf{p}_i, d, k, \mathcal{D}, t)$}
    \label{alg:dcfo}
    \end{algorithm}

    Thus, DCFO exploits the structure of the $\operatorname{LOF}$ function to generate high-quality counterfactuals using efficient, gradient-based optimisation. Algorithm~\ref{alg:dcfo} outlines how to find the counterfactual closest to the original outlier in terms of distance. A modified version of this algorithm, capable of producing multiple diverse counterfactuals, is described in Section~\ref{sec:multipleCounterfactuals}.

    An illustrative example of DCFO’s optimisation process is shown in Figure~\ref{fig:spacePartition}. The threshold for outlierness is set to $t=1.10$, and the red point represents the outlier $\mathbf{p}_i$ with an initial LOF score of $1.89$.

    The purple arrows show the optimisation steps of DCFO. The first solution (blue point) lies in a different region than the starting one, seen by the different background colours, but still violates the threshold ($\operatorname{LOF} = 1.11$). A second optimisation run (starting from that point) produces the final counterfactual (green point), which lies within the same region, and now satisfies the constraint ($\operatorname{LOF} = 1.10 = t$). This example demonstrates that DCFO does not require computing the full space partition in advance; instead, it only evaluates a small number of points.

    \subsection{Proposing Multiple Counterfactuals}\label{sec:multipleCounterfactuals}
        The procedure described in Algorithm~\ref{alg:dcfo} returns a single counterfactual. However, in many decision-making contexts, stakeholders may wish to consider multiple alternatives. Actionability depends on constraints or preferences varying across users, so offering a diverse set of counterfactuals improves the likelihood of identifying a feasible and meaningful proposal. 
    
        Moreover, simply generating many counterfactuals close to one another does not provide additional explanatory value. If counterfactuals are only infinitesimally different from one another, they are effectively indistinguishable in practice~\cite{ley2022diverse}. Prior work in explainability, similarly emphasises the importance of diversity~\cite{mothilal2020explaining, russell2019efficient}. 
        
        To generate multiple, diverse, and still proximal counterfactuals, we make use of the region-based partition of the input space. After identifying the optimal counterfactual using Algorithm~\ref{alg:dcfo}, we revisit the regions explored during the optimisation process. For each region not yet associated with an accepted counterfactual, we initiate a new constrained optimisation. Each new counterfactual is accepted if it belongs to a region that has not yet proposed any counterfactual, ensuring diversity by construction.
        
        The final result is a set of counterfactuals that are \emph{diverse}, since each counterfactual is based on a distinct key $\mathcal{K}(\mathbf{x})$, and \emph{proximal}, since each counterfactual is the result of a distance optimisation.

        If a user requests a fixed number $n$ of counterfactuals, the process stops once $n$ valid counterfactuals are found or the region queue is exhausted. As new regions are encountered during these secondary optimisations, they are added to the exploration list and considered if further counterfactuals are requested. The complete procedure is shown in the Appendix~\ref{sec:multiple_algo}, Algorithm~\ref{alg:dcfoMultiple}.

    \subsection{Non-Actionable Features}\label{subsec:nonActionable}
    In many real-world datasets, certain features are immutable or extremely difficult to change, such as ages. These are referred to as \emph{non-actionable features}~\cite{guidotti2024counterfactual, bhattacharya2025show, verma2024counterfactual}. If a counterfactual proposes changes to such features, it becomes infeasible for users to act upon, thereby invalidating its purpose. For instance, an outlier cannot be corrected if the suggested fix involves changing attributes that the user cannot modify.

    To address this, DCFO allows users to specify which features are actionable and which are not. The algorithm is modified accordingly: while distances and neighbourhoods are still computed in the whole feature space, the SLSQP optimiser is restricted to adjust only the actionable variables. The restriction ensures that non-actionable features remain fixed at their original values, and the counterfactual modifies only the permitted subset to maximise proximity.

    \subsection{Plausibility}\label{subsec:plausibility}
    Plausibility is a critical quality for counterfactuals, as proposed instances should be realistic and attainable. A common criticism of many counterfactual methods is that they produce solutions outside the empirical data distribution~\cite{zhou2025eace, kanamori2020dace, verma2020counterfactual}, making them likely infeasible or non-actionable. Specifically, if the counterfactual lies in a region of the feature space where no other instances are, it might imply that the region is challenging to reach, and the counterfactual may not be feasible.

    Fortunately, plausibility integrates naturally into DCFO, particularly because it operates on top of LOF. A widely accepted proxy for plausibility is data density: points that lie in high-density regions are more likely to be realistic and interpretable~\cite{kanamori2020dace}. Since LOF scores are inversely related to relative local density, high LOF values indicate low-density (i.e., isolated) regions, while low LOF values indicate high-density (i.e., typical) areas. Targeting a lower LOF score during optimisation naturally encourages plausibility. In other words, a point with a low LOF score lies in a region of the space with many similar instances. These regions are more likely to reflect realistic and feasible scenarios, as they represent common structures in the data. Therefore, by setting a stricter constraint on the LOF score (e.g., requiring $\operatorname{LOF}(\mathbf{x}) \le 1.25$ when the threshold is $t = 1.5$), DCFO can be steered toward generating counterfactuals that lie in denser, and hence more plausible areas of the feature space. In this way, DCFO not only provides proximity and diversity, but also generates counterfactuals that are grounded in realistic, high-density regions.
\section{Experiments}\label{sec:experiments}
    In this section, we empirically evaluate the quality of DCFO's counterfactual explanations using metrics from the counterfactual explanation literature~\cite{guidotti2024counterfactual}: proximity (closeness between point and explanation), validity (the proportion of valid counterfactuals), and diversity (how different the proposed counterfactuals are). In addition, we assess the robustness of DCFO when handling non-actionable features. Throughout all experiments, distances are computed using Euclidean distance, consistent with LOF's default implementation. Appendix~\ref{sec:runtime} details a full time complexity analysis.

    The experiments compare DCFO with three state-of-the-art methods. \textbf{(i)} Baycon, a model-agnostic counterfactual method that leverages Bayesian optimisation, handling non-differentiable models~\cite{romashov2022baycon}. \textbf{(ii)} EACE, a model-agnostic method explicitly designed for explaining outliers by generating counterfactuals close to the boundary separating outliers from inliers~\cite{zhou2025eace}. \textbf{(iii)} Baseline, a na\"ive approach that moves an outlier directly to the position of the nearest inlier.

    We considered other well-known model-agnostic methods such as DiCE~\cite{mothilal2020explaining} and FACE~\cite{poyiadzi2020face}. However, these methods failed to produce valid counterfactuals due to the complexity arising from LOF's non-continuos nature and outlier-specific decision boundaries (see Appendix~\ref{sec:dice&face} for details).

    Experiments are conducted using $50$ diverse datasets from the OpenML repository~\cite{vanschoren2014openml}, covering different domains, number of instances in the range $16:53940$, and dataset dimensionality in the range $2:124$~\footnote{Reproducible code for all experiments is available at our repository: https://anonymous.4open.science/r/DCFO-E37E/}. 

    \subsection{Quality of DCFO's Counterfactuals}\label{subsec:bestCounterfactual}
        We measure counterfactual quality in terms of proximity and validity. LOF is computed using Euclidean distance, and we randomly select the neighbourhood parameter ($k$) from the set: $\{10, 15, 20\}$, to demonstrate DCFO's robustness to LOF's parameter choice. 

        The LOF threshold, which separates inliers from outliers, is initially set to $1.5$, following standard practice~\cite{breunig2000lof}. If no points exceed this threshold, we set the threshold at the $95$th percentile of LOF scores, ensuring about $5\%$ of points are classified as outliers. 
        
        For each identified outlier, we run DCFO, Baycon, EACE and the Baseline method. LOF acts as a binary classifier, labelling points as outliers if their LOF scores exceeds the threshold. DCFO generates valid counterfactuals by ensuring the new LOF score is below this threshold. Baycon by default generates multiple counterfactual candidates.

        \paragraph{Validity} Counterfactuals are valid if the generated points are not identified as outliers. Thus, we report the proportion of outliers for which a counterfactual is identified for each dataset. A ´validity of $1$ means that the method can find a counterfactual for every outlier, while a validity of $0$ indicates that the method did not find counterfactuals.

        The results are summarised in Table~\ref{tab:validity}. Due to space constraints, we present $10$ of the $50$ datasets utilised in our analysis. The complete table can be found in Appendix~\ref{subsec:fullValidity}. Additionally, we report the validity for all $50$ datasets, $1.00$ for DCFO, $0.84$ for Baycon, $0.38$ for EACE, and $0.88$ for Baseline. 

        \begin{table}[tbp]
        \caption{Validity across all outliers for $10$ datasets. A score of $1$ indicates a valid counterfactual was found for every outlier. We do not report the standard error of the mean as we compute validity as a per-dataset and not per-instance metric.}
            \label{tab:validity}
            \centering
            \begin{tabular}{lrllll}
            \toprule
             & k & DCFO & Baycon & EACE & Baseline \\
            \midrule
            sonar & 10 & \textbf{1.00} & \textbf{1.00} & \textbf{1.00} & \textbf{1.00} \\
            liver-disorders & 10 & \textbf{1.00} & \textbf{1.00} & 0.56 & \textbf{1.00} \\
            diabetes & 10 & \textbf{1.00} & \textbf{1.00} & 0.33 & \textbf{1.00} \\
            triazines & 15 & \textbf{1.00} & \textbf{1.00} & 0.00 & 0.90 \\
            baskball & 15 & \textbf{1.00} & \textbf{1.00} & 0.00 & \textbf{1.00} \\
            glass & 15 & \textbf{1.00} & 0.56 & 0.39 & 0.75 \\
            ar3 & 15 & \textbf{1.00} & 0.73 & 0.09 & \textbf{1.00} \\
            disclosure x noise & 15 & \textbf{1.00} & \textbf{1.00} & 0.50 & \textbf{1.00} \\
            ionosphere & 10 & \textbf{1.00} & 0.15 & 0.13 & 0.41 \\
            chscase census2 & 10 & \textbf{1.00} & \textbf{1.00} & 0.50 & \textbf{1.00} \\
            \bottomrule
            \end{tabular}
        \end{table}

    Validity shows that DCFO is able to provide counterfactual explanations for every outlier detected, having validity $1$. On the contrary, other methods have lower validity, with EACE the weakest overall. 

    Through exploiting LOF's inner-workings, DCFO is able to propose a counterfactual for every outlier present in all the $50$ datasets studied. Due to the complexity introduced by an outlier detection task and LOF's non-continuous landscape, other methods fail instead to always provide counterfactuals, coming short of explaining all anomalous events. Note that Baseline's counterfactuals can also be invalid, as elaborated in Appendix~\ref{sec:baseline}.

    \paragraph{Proximity}
        Proximity denotes how similar the counterfactual is to the original data point, this is measured as the Euclidean distance between them. Thus, we select the closest valid counterfactual, i.e., a point with an LOF score below the threshold; if none are valid, Baycon's proposal is marked as invalid. Equivalently, if EACE does not provide any valid candidate, the validity for EACE is zero. The baseline directly moves the outlier to the nearest inlier, but does not guarantee validity due to LOF's relative density calculations, as elaborated in Appendix~\ref{sec:baseline}.

        Table~\ref{tab:proximity} reports proximity results for $10$ datasets. The average distance of the counterfactual with the lowest proximity for each outlier is highlighted in bold. If a method does not find any valid counterfactual for a given dataset, it's marked with Not Available (NA). Moreover, for each dataset we display LOF's $k$ parameter. The table containing all $50$ datasets can be found in Appendix~\ref{subsec:fullProximity}.
        
        As a result, we see in Table~\ref{tab:proximity} that DCFO consistently achieves superior proximity scores compared to other methods, demonstrating enhanced actionability and feasibility.

   \begin{table*}[t]
    \caption{Proximity results on $10$ representative datasets (from a total of $50$). For each dataset, we report the mean Euclidean distance between the original outliers and their closest valid counterfactuals along with the standard error of the mean. NA in the value for the standard error of the mean implies that the dataset has only one counterfactual, making the standard deviation impossible to computes. Lower values indicate better proximity and NA indicates no valid counterfactual found. LOF neighbourhood size is shown as $k$.}
    \label{tab:proximity}
    \centering
    \begin{tabular}{l r
                    r@{\,±\,}l
                    r@{\,±\,}l
                    r@{\,±\,}l
                    r@{\,±\,}l}
    \toprule
    Dataset & k
      & \multicolumn{2}{c}{DCFO}
      & \multicolumn{2}{c}{Baycon}
      & \multicolumn{2}{c}{EACE}
      & \multicolumn{2}{c}{Baseline} \\
    \cmidrule(lr){3-4}
    \cmidrule(lr){5-6}
    \cmidrule(lr){7-8}
    \cmidrule(lr){9-10}
    & 
      & mean & sem
      & mean & sem
      & mean & sem
      & mean & sem \\
    \midrule
    sonar & 10
      & $\mathbf{0.62}$ & \text{NA}
      & 1.59            & \text{NA}
      & 20.02           & \text{NA}
      & 9.14            & \text{NA} \\

    liver-disorders & 10
      & $\mathbf{0.15}$ & 0.04
      & 0.22            & 0.05
      & 1.3             & 0.2
      & 1.06            & 0.07 \\

    diabetes & 10
      & $\mathbf{0.29}$ & 0.05
      & 0.47            & 0.07
      & 1.7             & 0.2
      & 2.2             & 0.1 \\

    triazines & 15
      & $\mathbf{4.2}$  & 0.6
      & 7.9             & 0.7
      & \multicolumn{2}{c}{\text{NA}}
      & 10.7            & 0.7 \\

    baskball & 15
      & $\mathbf{1.16}$ & \text{NA}
      & 2.29            & \text{NA}
      & \multicolumn{2}{c}{\text{NA}}
      & 1.76            & \text{NA} \\

    glass & 15
      & $\mathbf{0.4}$  & 0.2
      & 0.7             & 0.4
      & 2.5             & 0.5
      & 1.4             & 0.4 \\

    ar3 & 15
      & $\mathbf{0.62}$ & \text{NA}
      & 0.74            & \text{NA}
      & 3.42            & \text{NA}
      & 2.33            & \text{NA} \\

    disclosure x noise & 15
      & $\mathbf{0.07}$ & 0.02
      & 0.15            & 0.03
      & 0.8             & 0.1
      & 0.81            & 0.05 \\

    ionosphere & 10
      & $\mathbf{0.2}$  & 0.1
      & 0.4             & 0.1
      & 6               & 1
      & 2.5             & 0.2 \\

    chscase census2 & 10
      & $\mathbf{0.30}$ & 0.05
      & 0.59            & 0.09
      & 1.9             & 0.2
      & 2.27            & 0.09 \\
    \bottomrule
    \end{tabular}
\end{table*}

        To facilitate comparison across datasets with varying dimensionalities, we utilise a Critical Difference (CD) diagram~\cite{demsar2006statistical}, seen in Figure~\ref{fig:cdPlot}. The CD diagram uses a Nemenyi post-hoc test to determine whether observed differences in average ranks are statistically significant.
        
        In our evaluation across all $50$ datasets, the methods are ranked according to their proximity performance, with rank $1$ assigned to the method that, on average, proposes the closest valid counterfactual for a given dataset, and rank $4$ to the one with the furthest. 
        
        \begin{figure}[tbp]
            \centering
            \includegraphics[width=1\linewidth]{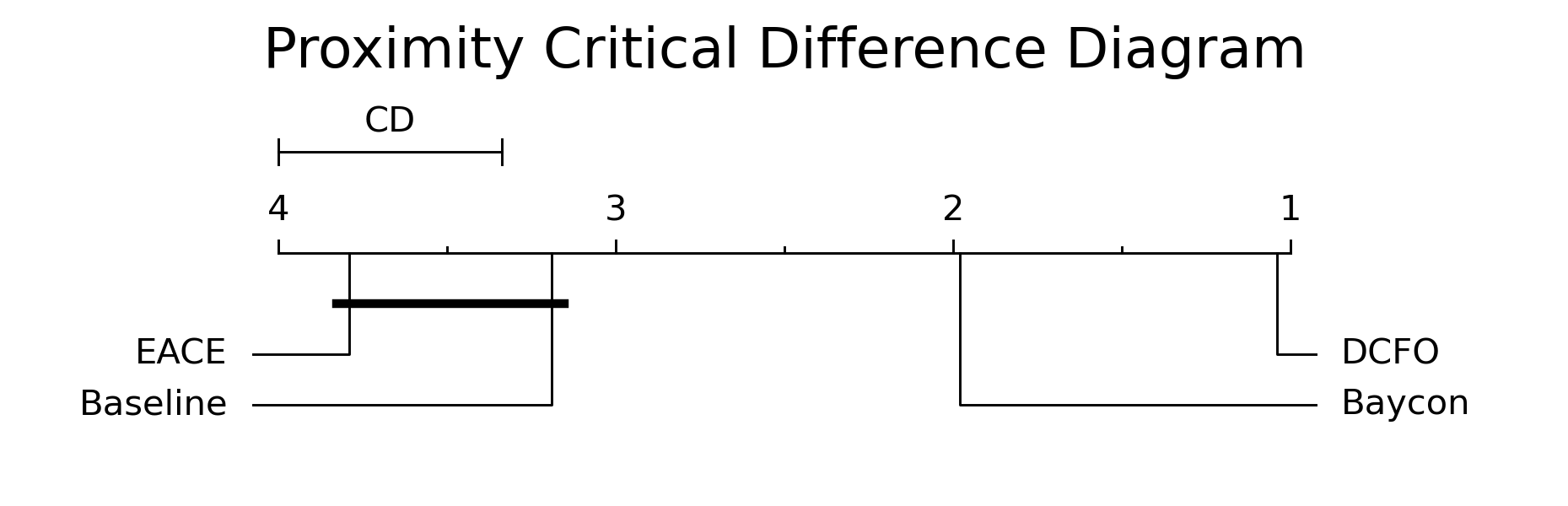}
            \caption{Proximity comparison across methods using a Critical Difference diagram. DCFO ranks first with statistically significant results (Nemenyi’s test). Baycon is second, while Baseline and EACE follow, though their differences are not significant.}
            \label{fig:cdPlot} 
            \Description{}
        \end{figure}
        
        In Figure~\ref{fig:cdPlot} the horizontal bar in the diagram denotes the minimum difference required for significance at $p < 0.05$. Methods whose average ranks are not significantly different are connected by a thick horizontal line. DCFO achieves the best average rank and, according to Nemenyi's test, the difference to the second-best method, Baycon, is statistically significant. Baycon ranks second, followed by the proposed Baseline and EACE. The difference between the Baseline and EACE is not statistically significant.
        
    \subsection{Diversity of Counterfactuals}\label{subsec:expMultple}
    In practical scenarios, providing multiple actionable alternatives is often desirable. To evaluate DCFO's capability to generate multiple diverse counterfactuals, we use the diversity metric from Mothilal et al.~\cite{mothilal2020explaining}. Specifically, we compute a kernel matrix $K_{ij} = 1/(1+d(\mathit{cf}_i, \mathit{cf}_j))$ and use its determinant as a measure of diversity: higher determinant values indicate greater diversity among proposed counterfactuals. For consistency, $10$ counterfactuals per instance are taken for each method, following their respective generative mechanisms. 

    We benchmark DCFO against Baycon and EACE, selecting for EACE the first 10 valid counterfactuals it provides. Baycon, lacking direct control over the number of counterfactuals proposed, is assessed in two ways: using the $10$ closest counterfactuals and a random selection of $10$ counterfactuals from its proposals. DCFO inherently prioritises proximity in its diversity generation, maintaining actionable solutions close to the original instance.

    To provide an interpretable assessment of diversity, we rank methods based on their determinant scores and visualise using the Critical Difference diagram (Figure~\ref{fig:cdPlotDiversity}). A score of zero denotes when a method fails to produce multiple counterfactuals, thus, ranking last. In Figure~\ref{fig:cdPlotDiversity}, we clearly see DCFO's statistically significant advantage in diversity compared to the other methods. Randomly selected counterfactuals from Baycon and EACE exhibit similar diversity, while Baycon’s $10$ closest solutions offer the least diversity.

    \begin{figure}[tbp]
        \centering
        \includegraphics[width=1\linewidth]{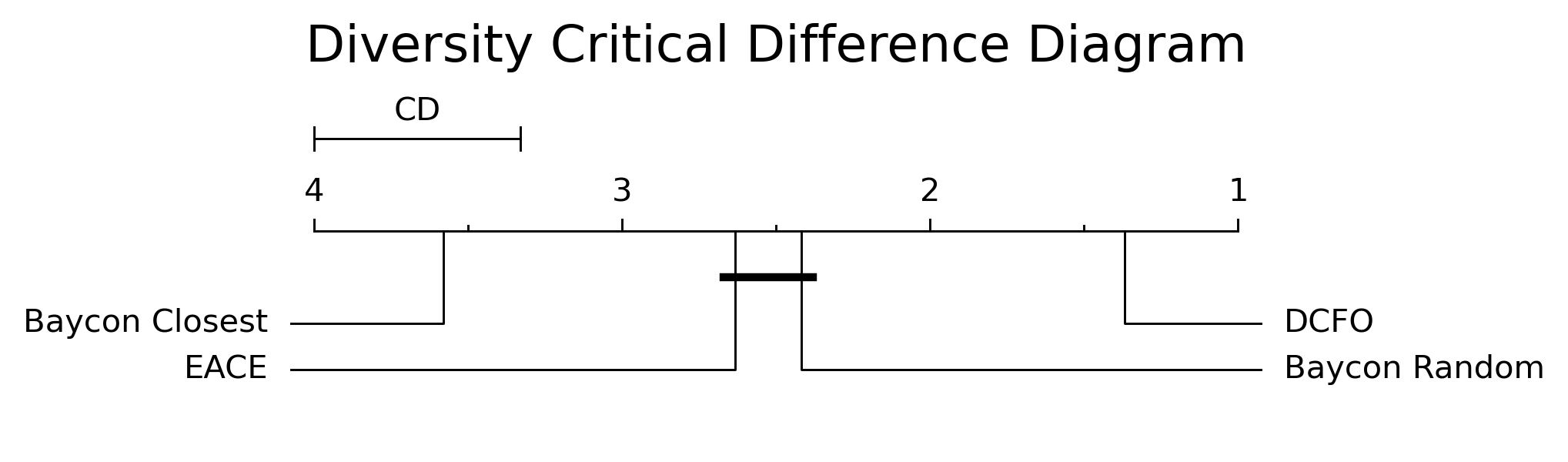}
        \caption{Critical Difference diagram comparing diversity among counterfactual generation methods. DCFO ranks highest with statistically significant differences.}
        \label{fig:cdPlotDiversity} 
        \Description{}
    \end{figure}

    \subsection{Impact of Non-Actionable Features on Counterfactual Generation}\label{subsec:expNonActionable}
    Real-world applications often limit the actionability of certain features. We evaluate DCFO's effectiveness when subsets of features are deemed non-actionable, measuring both proximity and validity across the selected datasets. Specifically, we randomly define a subset of non-actionable features for all $50$ datasets, ensuring they never exceed half the total features.

    Then, a further random process chooses which features cannot be altered. Once the set of non-actionable features is chosen, for each outlier, a counterfactual is proposed with both DCFO and Baycon: the two methods that can account for non-actionable features. Quality is assessed through closeness to the original instance (proximity) and proportion of valid counterfactuals (validity).

    Results for $10$ randomly chosen datasets are provided in Tables~\ref{tab:nonActionableProximity} and~\ref{tab:nanActionableValidity}, the column \%NAF reports the percentage of non-actionable features for the given dataset. For validity, we also report the average across al datasets which is $0.87$ for DCFO and $0.73$ for Baycon.
    
    The resulting tables demonstrate DCFO’s robustness, consistently outperforming Baycon in terms of proximity and validity; although some cases present increased complexity that reduces validity. This reduction highlights the challenge posed by feature constraints, underscoring DCFO’s superior capability in handling non-actionable feature scenarios compared to existing methods. The full table for proximity of all $50$ datasets is in Appendix~\ref{subsec:nonActFullProximity}, showing that DCFO outperforms Baycon in all but two datasets. The full validity table is in Appendix~\ref{subsec:nonActFullValidity}. 

\begin{table}[t]
    \caption{Proximity results for DCFO and Baycon when limiting the number of actionable features. $\%$NAF denotes the percentage of non-actionable features. We report the mean and its standard error across counterfactuals. NA as the standard error of the mean implies that only one counterfactual is present in the datset.}
    \label{tab:nonActionableProximity}
    \centering
    \begin{tabular}{l r r r@{\,±\,}l r@{\,±\,}l}
    \toprule
    Dataset & \%NAF & k
      & \multicolumn{2}{c}{DCFO}
      & \multicolumn{2}{c}{Baycon} \\
    \cmidrule(lr){4-5}
    \cmidrule(lr){6-7}
      &      &    
      & mean & sem
      & mean & sem \\
    \midrule
    confidence       & 0.33 & 20 & $\mathbf{0.2}$  & 0.1 & 0.2  & 0.1 \\
    wine             & 0.08 & 10 & $\mathbf{0.7}$  & 0.2 & 1.1  & 0.3 \\
    wine-quality-white & 0.27 & 20 & $\mathbf{1.1}$  & 0.2 & 1.6  & 0.2 \\
    wdbc             & 0.13 & 20 & $\mathbf{2.1}$  & 0.4 & 3.2  & 0.7 \\
    diggle table a1  & 0.25 & 15 & $\mathbf{0.14}$ & 0.01 & 0.2 & 0.1 \\
    libras move      & 0.28 & 20 & $\mathbf{0.61}$ & \text{NA} & 1.68 & \text{NA} \\
    kc3              & 0.26 & 15 & $\mathbf{0.8}$  & 0.2 & 1.2  & 0.3 \\
    kc2              & 0.14 & 10 & $\mathbf{0.8}$  & 0.2 & 1.3  & 0.3 \\
    mfeat-zernike    & 0.26 & 10 & $\mathbf{1.7}$  & 0.6 & 3.0  & 1.0 \\
    ecoli            & 0.14 & 15 & $\mathbf{2.1}$  & 0.5 & 2.4  & 0.5 \\
    \bottomrule
    \end{tabular}
\end{table}


    \begin{table}[t]
        \caption{Validity results for DCFO and Baycon when limiting the number of actionable features. $ \% NAF$ denoting the percentage of non-actionable features. The standard error of the mean is not reported as we compute validity as a per-dataset and not a per-instance metric.}
        \label{tab:nanActionableValidity}
\centering
        \begin{tabular}{lrrll}
        \toprule
         & \%NAF & k & DCFO & Baycon \\
        \midrule
        confidence & 0.33 & 20 & \textbf{1.00} & \textbf{1.00} \\
        wine & 0.08 & 10 & \textbf{1.00} & \textbf{1.00} \\
        wine-quality-white & 0.27 & 20 & \textbf{0.95} & 0.90 \\
        wdbc & 0.13 & 20 & \textbf{1.00} & 0.94 \\
        diggle table a1 & 0.25 & 15 & \textbf{1.00} & \textbf{1.00} \\
        libras move & 0.28 & 20 & \textbf{1.00} & \textbf{1.00} \\
        kc3 & 0.26 & 15 & \textbf{0.68} & 0.40 \\
        kc2 & 0.14 & 10 & \textbf{0.78} & 0.65 \\
        mfeat-zernike & 0.26 & 10 & \textbf{1.00} & \textbf{1.00} \\
        ecoli & 0.14 & 15 & \textbf{0.96} & 0.92 \\
        \bottomrule
        \end{tabular}
    \end{table}
\section{Conclusion and Future Work}\label{sec:conclusion}
Outlier detection is vital in many domains, yet explanations for why instances are flagged as outliers are often overlooked. Such explanations are essential to validate model outputs, understand their reasoning, and guide appropriate action.

We introduced DCFO, a counterfactual method designed to explain predictions made by LOF, a widely used density-based outlier detector. By leveraging LOF's internal structure, DCFO generates high-quality counterfactuals that correct outliers, account for non-actionable features, ensure plausibility, and offer multiple proposals. It does so by partitioning the input space into regions where LOF behaves more smoothly.


Through extensive experiments, we benchmarked DCFO on real-world datasets against competing methods. We evaluated counterfactual quality using proximity (closeness to the original point), validity (how often outputs are correct), and diversity (ability to propose multiple alternatives). Finally, we used proximity and validity again to assess DCFO’s handling of non-actionable features.

Future work can investigate extending DCFO to other variants of LOF~\cite{alghushairy2020review, kriegel2009loop,schubert14local}. Furthermore, DCFO could be extended with modifications that aim to minimise the number of features changed. 

\clearpage
\newpage
\bibliographystyle{ACM-Reference-Format}
\bibliography{ref}


\appendix
\appendixpage

\section{Multiple Counterfactual Algorithm}\label{sec:multiple_algo}
We present here how DCFO's algorithm is modified in order to propose multiple counterfactuals. After finding the optimal counterfactual, further optimisations are run from the points contained in the exploration list produced in the first search. New proposals are accepted if they belong to regions not yet having any counterfactual, ensuring diversity. Note that when looking for new counterfactuals, the exploration list initialised with the first optimisation is extended. In addition to Algorithm~\ref{alg:dcfo} inputs, the procedure to propose multiple counterfactuals, outlined in Algorithm~\ref{alg:dcfoMultiple}, accepts also the number of counterfactuals to propose in input.

 \begin{algorithm}[]
        \SetAlgoLined
        \Input{Outlier $\mathbf{p}_i$, distance function $d$, $\operatorname{LOF}$ parameter $k$, dataset $\mathcal{D}$, threshold $t$,\\ number of counterfactuals $n$.}
        \Output{Counterfactual List, $\operatorname{cfList}(\mathbf{p}_i)$}
        \vspace{1mm}
        Initialise counterfactual list, cfList $\leftarrow  \left [ \ \right ]$ \;
        \vspace{1mm}
        Initialise counterfactuals' regions list, regionsList $\leftarrow  \left [ \ \right ]$ \;
        \vspace{1mm}
        Let $cf(\mathbf{p}_i) \leftarrow \operatorname{DCFO}(\mathbf{p}_i, d, k, \mathcal{D}, y)$ \;
        \vspace{1mm}
        Let ExpList the exploration list from $\operatorname{DCFO}(\mathbf{p}_i, d, k, \mathcal{D}, y)$ \;
        \vspace{1mm}
        $\operatorname{cfList}.\operatorname{append}(cf(\mathbf{p}_i))$ \;
        \vspace{1mm}
        $\operatorname{regionsList}.\operatorname{append}(\mathcal{K}(cf(\mathbf{p}_i)))$ \;
        \vspace{1mm}
        \While{$\operatorname{length}(\operatorname{cfList}) < n$ and $\operatorname{ExpList} > 0$}{
        $\mathbf{x}_i \leftarrow \operatorname{ExpList}.\operatorname{pop}()$ \;
        $cf(\mathbf{x}_i) \leftarrow \operatorname{DCFO}(\mathbf{x}_i, d, k, \mathcal{D}, y)$ \;
        \If{
        $\mathcal{K}(cf(\mathbf{x})_i)\text{ not in }\operatorname{regionsList}$
        }{
        $\operatorname{regionsList}.\operatorname{append}(\mathcal{K}(cf(\mathbf{x})_i))$\;
        $\operatorname{cfList}.\operatorname{append}(cf(\mathbf{x})_i)$
        
        }
        }
        \caption{$\operatorname{DCFOmultiple}(\mathbf{p}_i, d, k, \mathcal{D}, t, n)$}
        \label{alg:dcfoMultiple}
    \end{algorithm}

\section{LOF's Differentiability}\label{sec:proof}


DCFO uses constraint gradient-based optimization to find counterfactuals. Thus, the function to be optimized and the constraint need to be differentiable. Specifically, SLSQP is an ideal solver for objective functions and constraints that are twice continuously differentiable. We assume that our objective function, i.e., the distance metric, is twice continuously differentiable. In particular, respecting common practice for LOF, we employ the Euclidean distance. Therefore, to use SLSQP, we are left to show that the constraint function, i.e., $\operatorname{LOF}_K(x)$, is twice continuously differentiable. In the following, the proposition already introduced in Section~\ref{subsec:dcfo} is proved.

\begin{proposition}
        Given a well-behaved distance function $d$, $\operatorname{LOF}_K(\mathbf{x})$ is continuous and differentiable almost everywhere in $\mathcal{X}$. Specifically, if $d$ is twice continuously differentiable, then $\operatorname{LOF}_K(\mathbf{x})$ is twice continuously differentiable almost everywhere.
    \end{proposition}

\begin{proof}
    Marking with $knn_K$ the set of neighbours of $\mathbf{x}$, which are contained in key $K$, we need to prove that 
    \begin{equation*}
        \operatorname{LOF}_K(\mathbf{x}) = \dfrac{1}{k \cdot lrd_k(\mathbf{x})} \cdot \sum\limits_{\mathbf{y}_i \in knn_K}lrd_k(\mathbf{y}_i) \quad ,
    \end{equation*}
is twice continuously differentiable in the variable $\mathbf{x}$. Exploiting the property that the sum of differentiable functions is still differentiable, and leaving out the constant from the differentiation process, we are left to prove that  
\begin{equation*}
        \operatorname{LOF}_K(\mathbf{x}) = \dfrac{lrd_k(\mathbf{y}_1)}{lrd_k(\mathbf{x})}   \quad 
\end{equation*}
is differentiable. Where, without loss of generality, we deal with one specific element of $knn_K$, i.e., $\mathbf{y}_1$. Note that in general $lrd_k(\mathbf{y}_1)$ does depend on $\mathbf{x}$, as $\mathbf{x}$ could be one of the neighbours of $\mathbf{y}_1$. If $\mathbf{x}$ is not one of the neighbours of $\mathbf{y}_1$, than $lrd_k(\mathbf{y}_1)$ is just a constant and proving that $lrd_k(\mathbf{x})$ is twice continuously differentiable will suffice.

We start by treating the denominator, i.e., $lrd_k(\mathbf{x})$:
\begin{equation*}
    lrd_k(\mathbf{x}) = \dfrac{k}{
        \sum\limits_{\mathbf{y}_i\in knn_K} rd_k(\mathbf{x}, \mathbf{y}_i)}
         \quad 
\end{equation*}
Exploiting again the properties of differentiable functions and choosing again, without loss of generality, $\mathbf{y}_i = \mathbf{y}_1$, we are left to look at
\begin{equation*}
        rd_k(\mathbf{x}, \mathbf{y}_1) = \max \left \{ k\text{-distance}(\mathbf{y}_1), d(\mathbf{x}, \mathbf{y}_1) \right \} \quad 
\end{equation*}

Given the assumptions on the distance function $d$, both terms in $k\text{-distance}(\mathbf{y}_1)$ and $d(\mathbf{x}, \mathbf{y}_1)$ are twice continuously differentiable.
In the all the points $\mathbf{z}_i$ such that $k\text{-distance}(\mathbf{y}_1) = d(\mathbf{x}, \mathbf{y}_1)$, $rd_k(\mathbf{x}, \mathbf{y}_1)$ is continuous, as both $k\text{-distance}(\mathbf{y}_1)$ and $d(\mathbf{x}, \mathbf{y}_1)$ have the same value in $\mathbf{z}_i$. $rd_k(\mathbf{x}, \mathbf{y}_1)$ though, is in general non-differentiable in $\mathbf{z}_i$, as the limit toward $\mathbf{z}_i$ of the derivative of $rd_k(\mathbf{x}, \mathbf{y}_1)$ might differ depending on from where $\mathbf{z}_i$ is approached. Critically, the set of points $\mathbf{z}_i$, has measure $0$. Indeed, it is generated from the equality constraint $k\text{-distance}(\mathbf{y}_1) - d(\mathbf{x}, \mathbf{y}_1) = 0$. If our dataset $\mathcal{D}$ lies in $\mathcal{R}^n$, the subset of points $\mathbf{z}_i$ will thus lie in a lower dimensional space $\mathcal{R}^{n-1}$. Thus, $rd_k(\mathbf{x}, \mathbf{y}_1)$ and consequently $\operatorname{LOF}_K(\mathbf{x})$, will be twice continuously differentiable except for a set of points with measure $0$, i.e., almost everywhere. 

\end{proof}

\section{DiCE and FACE}\label{sec:dice&face}

As outlier detection can be framed as a binary classification task, counterfactuals methods designed for classification can be employed. Specifically, DiCE is a popular solution aiming for feasible and diverse counterfactuals~\cite{zhou2025eace}.  DiCE builds on the work by Wachter et al.~\cite{wachter2017counterfactual}, adding diversity on top of proximity as the objective of the optimisation problem. Gradient descent is used to find the optimal solution of the defined loss function. When applied to our case study, LOF's non-differentiability prevents DiCE to find valid counterfactuals.

FACE instead, is an algorithm for counterfactuals focusing on finding feasible and actionable paths to cross the decision boundary. FACE, even though it does not rely on gradient decent as it does not optimize a loss function, requires a probabilistic predictor. In order to implement FACE, we thus need to convert LOF scores into outlying probabilities. The method chosen is the one presented by Kriegel et al.~\cite{kriegel2011interpreting}. The authors in~\cite{kriegel2011interpreting}, propose a method to unify arbitrary outlier factors into interpretable probability values.

After converting LOF scores to outlier probabilities, FACE still fails to propose valid counterfactuals. This can be attributed not only to LOF's being non-continuous, but also the the nature of the outlier detection task itself. FACE is indeed based on a proximity graph connecting different instances. Outliers can be and typically are, isolated, making difficult to build a density-connected path from the outlier instance to the inliers' class.

\section{Baseline's Validity}\label{sec:baseline}
To provide a simple counterfactual method, used to benchmark DCFO, we proposed Baseline. Baseline simply moves the outlier to the position of the closest inlier. Despite its simplicity, this intuitive heuristics does not guarantee that the proposed counterfactual has an LOF score below the set threshold. 

When moving the outlier to the inlier, the two new scores will coincide, but the new joint score might be higher than the original inlier's score, highlighting the challenge in proposing outliers for LOF.
The motivation originates directly from LOF's definition and is mainly twofold. It can be linked to both LOF looking at relative density (instead of density) and to LOF's relying on the reachability distance (instead of the distance). If the inlier's density increases, its relative density might decrease. Furthermore, despite a point moving closer to it, the inlier's density might decrease as it is not defined on a true metric but on the reachability distance. Below, a simple example shows a specific case where Baseline's counterfactual is not valid.

A toy dataset of $10$ instances is drawn from a two dimensional standard normal distribution. Then, LOF is fitted with $k=2$. The dataset along with the LOF scores are displayed on Figure~\ref{fig:baselineExample1}. The dotted lines represent the neighbourhood of a point, i.e., the circle having the point as the centre, and the point's $k$-distance as the radius. The outlier to analyse is marked in red (threshold $=1.5$), while the closest inlier is drawn in green.

First of all, it might seem unintuitive that the red point in Figure~\ref{fig:baselineExample1} has a higher LOF score than the green point. This happens because the red point's second neighbour is in a very high density region, having a low LOF value of $0.86$. This makes the red point's relative density lower than the green point's relative density.

\begin{figure}[h]
        \centering
        \includegraphics[width=1\linewidth]{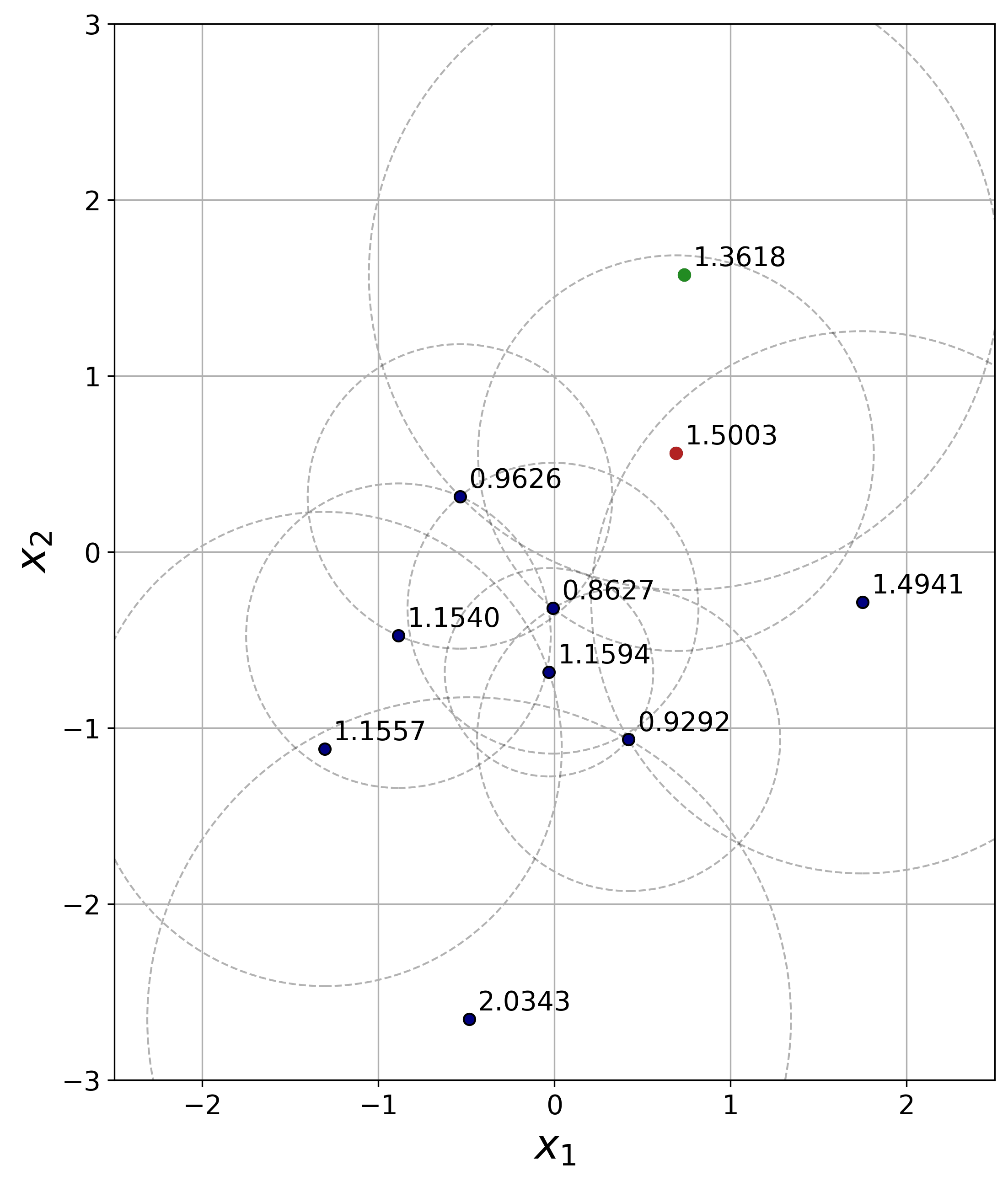}
        \caption{$2$D synthetic dataset drawn from a two dimensional standard normal distribution. The numbers next to the points represent their LOF score. In red the outlier analysed with its closest inlier in green.} 
        \label{fig:baselineExample1}
        \Description{}
    \end{figure}

We generate a counterfactual using the introduced Baseline method. The outlier is thus moved to the closest inlier, with the two points that will have a new common LOF score. Figure~\ref{fig:baselineExample2} shows the new setting, with the two superimposed points marked in red, having an LOF score of $1.56$, higher than the inlier's previous LOF score and above the set threshold.

\begin{figure}[h]
        \centering
        \includegraphics[width=1\linewidth]{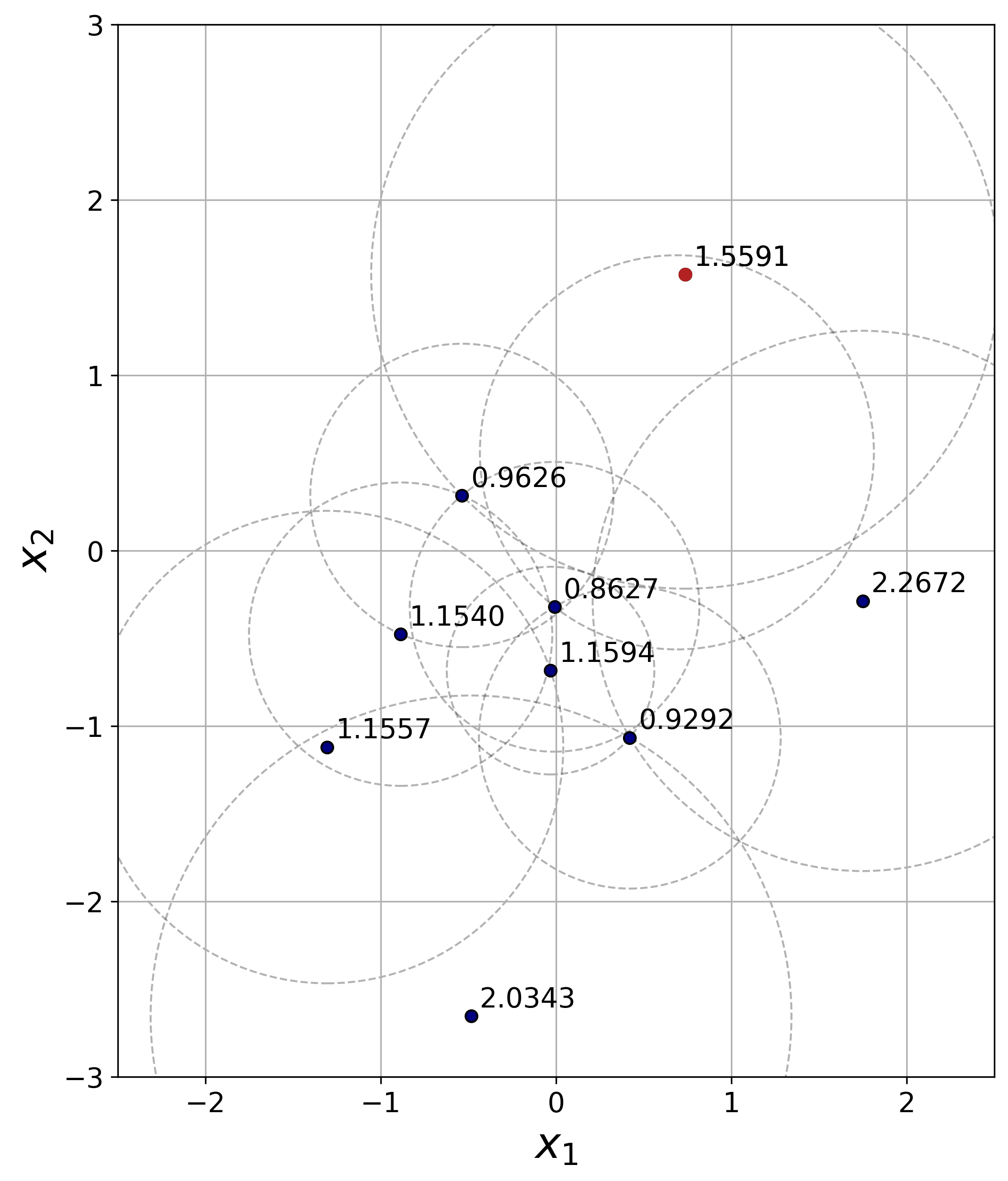}
        \caption{$2D$ synthetic dataset after the application of the Baseline method. The green and red points from Figure~\ref{fig:baselineExample1} are now superimposed (marked in red) and have an LOF score above the threshold.} 
        \label{fig:baselineExample2}
        \Description{}
    \end{figure}

The reason for the new LOF score being higher than that of the inlier can be mainly be attributed to the higher density (more precisely of the local reachability density) of the original inlier point in the new setting. Specifically, as written in Equation~\ref{eq:lrd}, the $lrd$ is the inverse of the average reachability distance of a point to its neighbours. In the case of the inlier, its two neighbours are the original outlier, which is now superimposed with the original inlier, and the point having an LOF score of $0.96$. 

Despite the outlier coming closer to the inlier, the reachability distance from the original inlier to the original outlier actually increases. This is because the outlier's $k$-distance increases, evident from the radius of the dotted circles, and it is higher than the Euclidean distance between the two points, which is $0$. The reachability distance to the second neighbour stays instead the same, as the two points do not move and their Euclidean distance is higher than the second neighbour's $k$-distance. Thus, the $lrd$ of the inlier decreases due to its $k$-distance to the closest neighbour (original outlier) increasing, causing the LOF to increase above the threshold.

\section{Counterfactuals' Empirical Evaluation}
Throughout the experiments, tables have been showcased for $10$ out of the total $50$ datasets employed, due to space reasons. Here, the complete results are displayed.

\subsection{Counterfactuals' Proximity}\label{subsec:fullProximity}
Section~\ref{subsec:bestCounterfactual}, shows how DCFO's counterfactuals score better in proximity with respect to competitors. Table~\ref{tab:fullProximity} expands the results covering all datasets tested. We can see how the trend already analysed in Table~\ref{tab:proximity}, is confirmed by Table~\ref{tab:fullProximity}. DCFO has indeed the best proximity among methods benchmarked across datasets. Only for two out of the $50$ datasets, i.e.\ \textit{longley} and \textit{chscase vine 1}, Baycon shows better proximity.


\begin{table*}[]
\caption{Results for proximity (lower is better) for all 50 datasets. Except for two datasets DCFO always performs best.}
\centering
\begin{tabular}{lccccc}
\hline
Dataset & k & DCFO & BayCon & EACE & Baseline \\
\hline
MeanWhile1 & 10 & \textbf{0.41 $\pm$ 0.06} & 0.83 $\pm$ 0.10 & 11 $\pm$ 1 & 5.0 $\pm$ 0.5 \\
MindCave2 & 15 & \textbf{0.4 $\pm$ 0.4} & 0.9 $\pm$ 0.8 & 9.95 $\pm$ 0.04 & 4.62 $\pm$ 0.06 \\
SPECTF & 20 & \textbf{0.37 $\pm$ 0.09} & 0.8 $\pm$ 0.2 & 9 $\pm$ 1 & 6.9 $\pm$ 0.2 \\
ar3 & 15 & \textbf{0.62 $\pm$ NA} & 0.74 $\pm$ NA & 3.42 $\pm$ NA & 2.34 $\pm$ NA \\
ar5 & 10 & \textbf{0.13 $\pm$ NA} & 0.23 $\pm$ NA & 5.86 $\pm$ NA & 3.22 $\pm$ NA \\
ar6 & 20 & \textbf{1 $\pm$ 1} & 2 $\pm$ 2 & 11 $\pm$ 5 & 4 $\pm$ 2 \\
autoPrice & 10 & \textbf{0.39 $\pm$ 0.05} & 0.63 $\pm$ 0.10 & 2.9 $\pm$ 0.3 & 3.1 $\pm$ 0.3 \\
baskball & 15 & \textbf{1.16 $\pm$ NA} & 2.29 $\pm$ NA & NA & 1.76 $\pm$ NA \\
blood-transfusion-service-center & 10 & \textbf{1.1 $\pm$ 0.2} & 1.1 $\pm$ 0.2 & 10 $\pm$ 2 & 1.4 $\pm$ 0.2 \\
bodyfat & 15 & \textbf{0.59 $\pm$ 0.10} & 0.9 $\pm$ 0.1 & 3.3 $\pm$ 0.5 & 3.5 $\pm$ 0.1 \\
chscase census2 & 10 & \textbf{0.30 $\pm$ 0.05} & 0.59 $\pm$ 0.09 & 1.9 $\pm$ 0.2 & 2.27 $\pm$ 0.09 \\
chscase census6 & 10 & \textbf{0.23 $\pm$ 0.07} & 0.8 $\pm$ 0.1 & 1.5 $\pm$ 0.3 & 1.56 $\pm$ 0.09 \\
chscase vine1 & 10 & 0.7 $\pm$ 0.2 & \textbf{0.5 $\pm$ 0.2} & 2.2 $\pm$ 0.2 & NA \\
confidence & 10 & \textbf{0.19 $\pm$ 0.05} & 0.22 $\pm$ 0.05 & 6.7 $\pm$ 0.7 & 0.70 $\pm$ 0.07 \\
diabetes & 10 & \textbf{0.29 $\pm$ 0.05} & 0.47 $\pm$ 0.07 & 1.7 $\pm$ 0.2 & 2.2 $\pm$ 0.1 \\
diabetes numeric & 10 & \textbf{0.59 $\pm$ NA} & 1.07 $\pm$ NA & NA & 1.55 $\pm$ NA \\
diamonds & 10 & \textbf{0.079 $\pm$ 0.006} & 0.126 $\pm$ 0.008 & 3.52 $\pm$ 0.07 & 0.208 $\pm$ 0.008 \\
diggle table a1 & 10 & \textbf{0.29 $\pm$ 0.07} & 0.35 $\pm$ 0.08 & NA & 1.796 $\pm$ 0.002 \\
disclosure x noise & 15 & \textbf{0.07 $\pm$ 0.02} & 0.15 $\pm$ 0.03 & 0.8 $\pm$ 0.1 & 0.81 $\pm$ 0.05 \\
ecoli & 15 & \textbf{0.17 $\pm$ 0.02} & 0.22 $\pm$ 0.01 & 1.2 $\pm$ 0.1 & 1.05 $\pm$ 0.04 \\
glass & 15 & \textbf{0.4 $\pm$ 0.2} & 0.7 $\pm$ 0.4 & 2.5 $\pm$ 0.5 & 1.4 $\pm$ 0.4 \\
heart-statlog & 20 & \textbf{0.05 $\pm$ NA} & 0.12 $\pm$ NA & 5.43 $\pm$ NA & 4.04 $\pm$ NA \\
ionosphere & 10 & \textbf{0.2 $\pm$ 0.1} & 0.4 $\pm$ 0.1 & 6 $\pm$ 1 & 2.5 $\pm$ 0.2 \\
iris & 15 & \textbf{0.07 $\pm$ NA} & 0.13 $\pm$ NA & 1.21 $\pm$ NA & 0.45 $\pm$ NA \\
kc1-top5 & 10 & \textbf{0.6 $\pm$ 0.2} & 3 $\pm$ 2 & 7 $\pm$ 2 & 9 $\pm$ 1 \\
kc2 & 10 & \textbf{0.39 $\pm$ 0.08} & 0.58 $\pm$ 0.09 & 3.0 $\pm$ 0.4 & 1.1 $\pm$ 0.1 \\
kc3 & 20 & \textbf{0.7 $\pm$ 0.1} & 0.9 $\pm$ 0.1 & 7.9 $\pm$ 0.5 & 2.6 $\pm$ 0.2 \\
libras move & 10 & \textbf{0.2 $\pm$ 0.1} & 0.9 $\pm$ 0.3 & 4 $\pm$ 2 & 6 $\pm$ 2 \\
liver-disorders & 10 & \textbf{0.15 $\pm$ 0.04} & 0.22 $\pm$ 0.05 & 1.3 $\pm$ 0.2 & 1.06 $\pm$ 0.07 \\
longley & 10 & 2.7 $\pm$ 0.6 & \textbf{1.7 $\pm$ 0.4} & NA & NA \\
machine cpu & 10 & \textbf{0.14 $\pm$ 0.08} & 0.29 $\pm$ 0.06 & 2 $\pm$ 1 & 1.03 $\pm$ 0.08 \\
mfeat-zernike & 15 & \textbf{0.9 $\pm$ 0.5} & 2 $\pm$ 1 & 14 $\pm$ 5 & 7 $\pm$ 3 \\
mu284 & 10 & \textbf{0.2 $\pm$ 0.1} & 0.6 $\pm$ 0.2 & 1.7 $\pm$ 0.5 & 0.96 $\pm$ 0.01 \\
no2 & 10 & \textbf{0.26 $\pm$ 0.08} & 0.33 $\pm$ 0.09 & 1.8 $\pm$ 0.4 & 1.7 $\pm$ 0.3 \\
pm10 & 15 & \textbf{0.02 $\pm$ NA} & 0.08 $\pm$ NA & 2.09 $\pm$ NA & 1.77 $\pm$ NA \\
prnn fglass & 10 & \textbf{0.24 $\pm$ 0.04} & 0.39 $\pm$ 0.07 & 2.7 $\pm$ 0.3 & 1.3 $\pm$ 0.3 \\
pyrim & 15 & \textbf{1.10 $\pm$ NA} & 3.35 $\pm$ NA & 6.28 $\pm$ NA & 8.66 $\pm$ NA \\
rabe 131 & 10 & \textbf{2.29 $\pm$ NA} & 3.09 $\pm$ NA & NA & 4.84 $\pm$ NA \\
sleep & 20 & \textbf{2.4 $\pm$ 0.6} & 3.2 $\pm$ 1.0 & NA & 4 $\pm$ 1 \\
sonar & 10 & \textbf{0.62 $\pm$ NA} & 1.59 $\pm$ NA & 20.02 $\pm$ NA & 9.14 $\pm$ NA \\
steel-plates-fault & 10 & \textbf{0.54 $\pm$ 0.10} & 0.7 $\pm$ 0.1 & 11 $\pm$ 1 & 3.1 $\pm$ 0.3 \\
strikes & 15 & \textbf{0.4 $\pm$ 0.3} & 0.5 $\pm$ 0.3 & 2 $\pm$ 1 & 1.3 $\pm$ 0.4 \\
tecator & 15 & \textbf{0.8 $\pm$ 0.3} & 2.4 $\pm$ 0.9 & NA & 6.9 $\pm$ 0.8 \\
triazines & 15 & \textbf{4.2 $\pm$ 0.6} & 7.9 $\pm$ 0.7 & NA & 10.7 $\pm$ 0.7 \\
vehicle & 15 & \textbf{1.1 $\pm$ 0.1} & 1.3 $\pm$ 0.1 & NA & 3.30 $\pm$ 0.10 \\
wdbc & 10 & \textbf{0.58 $\pm$ 0.05} & 1.20 $\pm$ 0.08 & 5.8 $\pm$ 0.3 & 4.7 $\pm$ 0.2 \\
wine & 10 & \textbf{0.21 $\pm$ 0.09} & 0.4 $\pm$ 0.2 & 4.0 $\pm$ 0.5 & 2.4 $\pm$ 0.5 \\
wine-quality-red & 20 & \textbf{0.26 $\pm$ 0.05} & 0.50 $\pm$ 0.09 & 3.4 $\pm$ 0.4 & 2.2 $\pm$ 0.2 \\
wine-quality-white & 10 & \textbf{0.39 $\pm$ 0.03} & 0.66 $\pm$ 0.04 & 3.6 $\pm$ 0.1 & 2.15 $\pm$ 0.05 \\
wisconsin & 10 & \textbf{0.36 $\pm$ 0.07} & 1.0 $\pm$ 0.2 & 7 $\pm$ 2 & 6.3 $\pm$ 0.2 \\
\bottomrule
\end{tabular}
\label{tab:fullProximity}
\end{table*}

\subsection{Counterfactuals' Validity}\label{subsec:fullValidity}
Table~\ref{tab:validity} is complemented here to include validity results for all $50$ datasets. Table~\ref{tab:fullValidity} shows DCFO being able to provide a valid counterfactual in all cases, as the average validity is $1$ for all datasets. Looking at average validity, Baseline scores second-best while EACE shows the worst validity behind Baycon.

\begin{table}[]
\caption{Results for validity (higher is better) for all 50 datasets. DCFO is the only method able to provide a valid counterfactual for all outliers.}
\centering
\begin{tabular}{lrllll}
\toprule
 & k & DCFO & Baycon & EACE & Baseline \\
\midrule
MeanWhile1 & 10 & \textbf{1.00} & 0.67 & 0.40 & \textbf{1.00} \\
MindCave2 & 15 & \textbf{1.00} & 0.88 & 0.25 & \textbf{1.00} \\
SPECTF & 20 & \textbf{1.00} & 0.88 & 0.27 & 0.96 \\
ar3 & 15 & \textbf{1.00} & 0.73 & 0.09 & \textbf{1.00} \\
ar5 & 10 & \textbf{1.00} & \textbf{1.00} & 0.33 & \textbf{1.00} \\
ar6 & 20 & \textbf{1.00} & 0.60 & 0.40 & \textbf{1.00} \\
autoPrice & 10 & \textbf{1.00} & 0.81 & 0.62 & \textbf{1.00} \\
baskball & 15 & \textbf{1.00} & \textbf{1.00} & 0.00 & \textbf{1.00} \\
blood-transfusion & 10 & \textbf{1.00} & 0.57 & 0.43 & 0.93 \\
bodyfat & 15 & \textbf{1.00} & \textbf{1.00} & 0.55 & 0.91 \\
chscase census2 & 10 & \textbf{1.00} & \textbf{1.00} & 0.50 & \textbf{1.00} \\
chscase census6 & 10 & \textbf{1.00} & \textbf{1.00} & 0.23 & \textbf{1.00} \\
chscase vine1 & 10 & \textbf{1.00} & 0.57 & 0.16 & 0.00 \\
confidence & 10 & \textbf{1.00} & \textbf{1.00} & \textbf{1.00} & \textbf{1.00} \\
diabetes & 10 & \textbf{1.00} & \textbf{1.00} & 0.33 & \textbf{1.00} \\
diabetes numeric & 10 & \textbf{1.00} & \textbf{1.00} & 0.00 & 0.50 \\
diamonds & 10 & \textbf{1.00} & 0.71 & 0.98 & 0.99 \\
diggle table a1 & 10 & \textbf{1.00} & \textbf{1.00} & 0.00 & \textbf{1.00} \\
disclosure x noise & 15 & \textbf{1.00} & \textbf{1.00} & 0.50 & \textbf{1.00} \\
ecoli & 15 & \textbf{1.00} & 0.88 & 0.27 & 0.96 \\
glass & 15 & \textbf{1.00} & 0.56 & 0.39 & 0.75 \\
heart-statlog & 20 & \textbf{1.00} & \textbf{1.00} & 0.50 & \textbf{1.00} \\
ionosphere & 10 & \textbf{1.00} & 0.15 & 0.13 & 0.41 \\
iris & 15 & \textbf{1.00} & \textbf{1.00} & 0.17 & 0.83 \\
kc1-top5 & 10 & \textbf{1.00} & 0.90 & 0.20 & 0.85 \\
kc2 & 10 & \textbf{1.00} & 0.73 & 0.47 & 0.73 \\
kc3 & 20 & \textbf{1.00} & 0.71 & 0.48 & 0.75 \\
libras move & 10 & \textbf{1.00} & 0.67 & 0.67 & \textbf{1.00} \\
liver-disorders & 10 & \textbf{1.00} & \textbf{1.00} & 0.56 & \textbf{1.00} \\
longley & 10 & \textbf{1.00} & 0.60 & 0.00 & 0.00 \\
machine cpu & 10 & \textbf{1.00} & \textbf{1.00} & 0.21 & 0.88 \\
mfeat-zernike & 15 & \textbf{1.00} & \textbf{1.00} & 0.67 & \textbf{1.00} \\
mu284 & 10 & \textbf{1.00} & 0.62 & 0.25 & \textbf{1.00} \\
no2 & 10 & \textbf{1.00} & \textbf{1.00} & 0.83 & \textbf{1.00} \\
pm10 & 15 & \textbf{1.00} & \textbf{1.00} & \textbf{1.00} & \textbf{1.00} \\
prnn fglass & 10 & \textbf{1.00} & 0.54 & 0.57 & 0.89 \\
pyrim & 15 & \textbf{1.00} & \textbf{1.00} & 0.25 & \textbf{1.00} \\
rabe 131 & 10 & \textbf{1.00} & \textbf{1.00} & 0.00 & \textbf{1.00} \\
sleep & 20 & \textbf{1.00} & 0.67 & 0.00 & \textbf{1.00} \\
sonar & 10 & \textbf{1.00} & \textbf{1.00} & \textbf{1.00} & \textbf{1.00} \\
steel-plates-fault & 10 & \textbf{1.00} & 0.45 & 0.62 & 0.98 \\
strikes & 15 & \textbf{1.00} & 0.90 & 0.30 & \textbf{1.00} \\
tecator & 15 & \textbf{1.00} & 0.71 & 0.00 & 0.86 \\
triazines & 15 & \textbf{1.00} & \textbf{1.00} & 0.00 & 0.90 \\
vehicle & 15 & \textbf{1.00} & 0.33 & 0.00 & \textbf{1.00} \\
wdbc & 10 & \textbf{1.00} & 0.96 & 0.59 & 0.96 \\
wine & 10 & \textbf{1.00} & \textbf{1.00} & 0.44 & 0.78 \\
wine-quality-red & 20 & \textbf{1.00} & 0.93 & 0.44 & 0.98 \\
wine-quality-white & 10 & \textbf{1.00} & 0.89 & 0.71 & 0.96 \\
wisconsin & 10 & \textbf{1.00} & \textbf{1.00} & 0.40 & \textbf{1.00} \\
\bottomrule
Average & - & \textbf{1.00} & 0.84 & 0.38 & 0.88\\
\end{tabular}
\label{tab:fullValidity}
\end{table}

\subsection{Non-actionable Features: Proximity}\label{subsec:nonActFullProximity}

We show, in Table~\ref{tab:nonActFullProximity}, proximity results for all $50$ OpenMl datasets when a percentage of the features are non-actionable. DCFO is the better performing model except for \textit{longley} and \textit{blood-transfusion}, confirming its superior performance in terms of proximity even when a subset of the features cannot be modified.

\begin{table*}[]
\caption{Results for proximity (lower is better) for all 50 datasets when a percentage of the features is deemed non-actionable. DCFO has better proximity than Baycon for all datasets but \textit{longley} and \textit{blood-transfusion}. The \%NAF column reports the percentage of non-actionable features for the given dataset.}
\centering
\begin{tabular}{lccc}
\hline
Dataset & k & DCFO & Baycon \\
\hline
MeanWhile1 & 15 & \textbf{1.4 $\pm$ 0.6} & 2 $\pm$ 1 \\
MindCave2 & 10 & \textbf{0.4 $\pm$ 0.2} & 0.8 $\pm$ 0.5 \\
SPECTF & 15 & \textbf{1.1 $\pm$ 0.2} & 2.5 $\pm$ 0.5 \\
ar3 & 10 & \textbf{1.5 $\pm$ 0.2} & 2.1 $\pm$ 0.4 \\
ar5 & 20 & \textbf{12.51 $\pm$ NA} & NA \\
ar6 & 10 & \textbf{0.9 $\pm$ 0.6} & 3.2 $\pm$ 1.0 \\
autoPrice & 20 & \textbf{1.1 $\pm$ 0.1} & 1.5 $\pm$ 0.2 \\
baskball & 10 & \textbf{0.6 $\pm$ 0.4} & 0.9 $\pm$ 0.7 \\
blood-transfusion-service-center & 10 & 2.3 $\pm$ 0.8 & \textbf{2.3 $\pm$ 0.8} \\
bodyfat & 10 & \textbf{2 $\pm$ 1} & 2 $\pm$ 1 \\
chscase census2 & 20 & \textbf{0.8 $\pm$ 0.2} & 1.5 $\pm$ 0.3 \\
chscase census6 & 15 & \textbf{1.9 $\pm$ 0.4} & 2.7 $\pm$ 0.5 \\
chscase vine1 & 20 & \textbf{0.95 $\pm$ NA} & 1.19 $\pm$ NA \\
confidence & 20 & \textbf{0.2 $\pm$ 0.1} & 0.2 $\pm$ 0.1 \\
diabetes & 10 & \textbf{0.8 $\pm$ 0.2} & 1.3 $\pm$ 0.3 \\
diabetes numeric & 15 & \textbf{0.15 $\pm$ NA} & 0.18 $\pm$ NA \\
diamonds & 10 & \textbf{0.084 $\pm$ 0.007} & 0.130 $\pm$ 0.010 \\
diggle table a1 & 15 & \textbf{0.14 $\pm$ 0.01} & 0.2 $\pm$ 0.1 \\
disclosure x noise & 15 & \textbf{0.27 $\pm$ 0.07} & 0.5 $\pm$ 0.1 \\
ecoli & 15 & \textbf{2.1 $\pm$ 0.5} & 2.4 $\pm$ 0.5 \\
glass & 10 & \textbf{0.7 $\pm$ 0.2} & 0.9 $\pm$ 0.3 \\
heart-statlog & 15 & \textbf{0.6 $\pm$ 0.5} & 0.8 $\pm$ 0.5 \\
ionosphere & 10 & \textbf{0.5 $\pm$ 0.1} & 0.8 $\pm$ 0.2 \\
iris & 15 & \textbf{0.5 $\pm$ 0.1} & 0.6 $\pm$ 0.1 \\
kc1-top5 & 10 & \textbf{2.4 $\pm$ 0.8} & 5 $\pm$ 1 \\
kc2 & 10 & \textbf{0.8 $\pm$ 0.2} & 1.3 $\pm$ 0.3 \\
kc3 & 15 & \textbf{0.8 $\pm$ 0.2} & 1.2 $\pm$ 0.3 \\
libras move & 20 & \textbf{0.61 $\pm$ NA} & 1.68 $\pm$ NA \\
liver-disorders & 10 & \textbf{0.7 $\pm$ 0.2} & 1.6 $\pm$ 0.5 \\
longley & 10 & 2.6 $\pm$ 0.6 & \textbf{2.1 $\pm$ 0.6} \\
machine cpu & 20 & \textbf{1.4 $\pm$ 0.3} & 1.7 $\pm$ 0.3 \\
mfeat-zernike & 10 & \textbf{1.7 $\pm$ 0.6} & 3 $\pm$ 1 \\
mu284 & 15 & \textbf{0.39 $\pm$ 0.08} & 0.6 $\pm$ 0.1 \\
no2 & 15 & \textbf{0.23 $\pm$ 0.09} & 0.25 $\pm$ 0.08 \\
pm10 & 15 & \textbf{0.02 $\pm$ NA} & 0.06 $\pm$ NA \\
pyrim & 20 & \textbf{2 $\pm$ 1} & 3 $\pm$ 1 \\
rabe 131 & 10 & \textbf{2.43 $\pm$ NA} & 3.07 $\pm$ NA \\
sleep & 15 & \textbf{2.9 $\pm$ 0.9} & 3.3 $\pm$ 0.9 \\
sonar & 15 & \textbf{0.97 $\pm$ NA} & 3.43 $\pm$ NA \\
steel-plates-fault & 15 & \textbf{0.56 $\pm$ 0.09} & 1.0 $\pm$ 0.2 \\
strikes & 10 & \textbf{0.8 $\pm$ 0.2} & 0.9 $\pm$ 0.2 \\
tecator & 20 & \textbf{1.7 $\pm$ 0.5} & 4 $\pm$ 1 \\
triazines & 10 & \textbf{3.5 $\pm$ 0.5} & 6.1 $\pm$ 0.6 \\
vehicle & 20 & \textbf{6 $\pm$ 2} & 6 $\pm$ 2 \\
wdbc & 20 & \textbf{2.1 $\pm$ 0.4} & 3.2 $\pm$ 0.7 \\
wine & 10 & \textbf{0.7 $\pm$ 0.2} & 1.1 $\pm$ 0.3 \\
wine-quality-red & 10 & \textbf{0.7 $\pm$ 0.1} & 1.3 $\pm$ 0.3 \\
wine-quality-white & 20 & \textbf{1.1 $\pm$ 0.2} & 1.6 $\pm$ 0.2 \\
wisconsin & 10 & \textbf{2 $\pm$ 1} & 3 $\pm$ 1 \\
\hline
\end{tabular}
\label{tab:nonActFullProximity}
\end{table*}

\subsection{Non-actionable Features: Validity}\label{subsec:nonActFullValidity}

Table~\ref{tab:nonActFullValidity} reports validity results for all datasets when a percentage of the features are non-actionable. DCFO shows higher validity than Baycon in all datasets. Compared with the results in Table~\ref{tab:fullValidity}, validity is on average lower in several datasets, for both methods. This underscores the challenges of dealing with non-actionable features, as the added constraints limit the space in which the respective optimisers can move. 

\begin{table}[]
\caption{Results for validity (higher is better) for all 50 datasets when a percentage of the features is deemed non-actionable. DCFO is able to provide the same or a higher percentage of valid counterfactuals for all datasets. The \%NAF column reports the percentage of non-actionable features for the given dataset.}
\centering
\begin{tabular}{llrll}
\toprule
 & \%NAF & k & DCFO & Baycon \\
\midrule
MeanWhile1 & 0.27 & 15 & \textbf{0.73} & 0.47 \\
MindCave2 & 0.31 & 10 & \textbf{0.71} & 0.43 \\
SPECTF & 0.11 & 15 & \textbf{1.00} & 0.76 \\
ar3 & 0.07 & 10 & \textbf{1.00} & 0.56 \\
ar5 & 0.10 & 20 & \textbf{1.00} & 0.00 \\
ar6 & 0.10 & 10 & \textbf{0.50} & \textbf{0.50} \\
autoPrice & 0.13 & 20 & \textbf{1.00} & \textbf{1.00} \\
baskball & 0.25 & 10 & \textbf{1.00} & \textbf{1.00} \\
blood-transfusion & 0.25 & 10 & \textbf{0.93} & 0.57 \\
bodyfat & 0.29 & 10 & \textbf{0.73} & \textbf{0.73} \\
chscase census2 & 0.29 & 20 & \textbf{0.81} & \textbf{0.81} \\
chscase census6 & 0.17 & 15 & \textbf{0.87} & \textbf{0.87} \\
chscase vine1 & 0.33 & 20 & \textbf{1.00} & \textbf{1.00} \\
confidence & 0.33 & 20 & \textbf{1.00} & \textbf{1.00} \\
diabetes & 0.12 & 10 & \textbf{0.83} & 0.75 \\
diabetes numeric & 0.50 & 15 & \textbf{0.50} & \textbf{0.50} \\
diamonds & 0.11 & 10 & \textbf{0.71} & 0.46 \\
diggle table a1 & 0.25 & 15 & \textbf{1.00} & \textbf{1.00} \\
disclosure x noise & 0.33 & 15 & \textbf{0.89} & \textbf{0.89} \\
ecoli & 0.14 & 15 & \textbf{0.96} & 0.92 \\
glass & 0.33 & 10 & \textbf{0.65} & 0.49 \\
heart-statlog & 0.31 & 15 & \textbf{1.00} & \textbf{1.00} \\
ionosphere & 0.06 & 10 & \textbf{0.72} & 0.13 \\
iris & 0.25 & 15 & \textbf{1.00} & \textbf{1.00} \\
kc1-top5 & 0.44 & 10 & \textbf{0.55} & 0.40 \\
kc2 & 0.14 & 10 & \textbf{0.78} & 0.65 \\
kc3 & 0.26 & 15 & \textbf{0.68} & 0.40 \\
libras move & 0.28 & 20 & \textbf{1.00} & \textbf{1.00} \\
liver-disorders & 0.20 & 10 & \textbf{0.94} & \textbf{0.94} \\
longley & 0.17 & 10 & \textbf{0.73} & 0.47 \\
machine cpu & 0.17 & 20 & \textbf{0.80} & 0.77 \\
mfeat-zernike & 0.26 & 10 & \textbf{1.00} & \textbf{1.00} \\
mu284 & 0.22 & 15 & \textbf{0.62} & \textbf{0.62} \\
no2 & 0.29 & 15 & \textbf{1.00} & \textbf{1.00} \\
pm10 & 0.29 & 15 & \textbf{1.00} & \textbf{1.00} \\
pyrim & 0.33 & 20 & \textbf{1.00} & 0.67 \\
rabe 131 & 0.20 & 10 & \textbf{1.00} & \textbf{1.00} \\
sleep & 0.14 & 15 & \textbf{0.67} & \textbf{0.67} \\
sonar & 0.45 & 15 & \textbf{1.00} & \textbf{1.00} \\
steel-plates-fault & 0.18 & 15 & \textbf{0.88} & 0.52 \\
strikes & 0.17 & 10 & \textbf{1.00} & 0.91 \\
tecator & 0.28 & 20 & \textbf{0.92} & 0.69 \\
triazines & 0.27 & 10 & \textbf{0.65} & 0.61 \\
vehicle & 0.28 & 20 & \textbf{1.00} & 0.58 \\
wdbc & 0.13 & 20 & \textbf{1.00} & 0.94 \\
wine & 0.08 & 10 & \textbf{1.00} & \textbf{1.00} \\
wine-quality-red & 0.36 & 10 & \textbf{0.89} & 0.75 \\
wine-quality-white & 0.27 & 20 & \textbf{0.95} & 0.90 \\
wisconsin & 0.47 & 10 & \textbf{0.90} & 0.60 \\
\bottomrule
Average & - & - & \textbf{0.87} & 0.73 \\
\end{tabular}
\label{tab:nonActFullValidity}
\end{table}

\section{Run Time Analysis}\label{sec:runtime}
Despite LOF being a non-continuous function, DCFO manages to use gradient-based optimisation, which is generically faster than Bayesian and genetic algorithms. In this section, we benchmark DCFO's run time against competitors. Table~\ref{tab:runTime} showcases average run times (in seconds) across outliers for all datasets. Because different datasets differ in complexity, we do not take an average, but an aggregation of results across datasets is carried out through a critical difference diagram, shown in Figure~\ref{fig:cdPlotRunTime}. For DCFO, EACE and Baseline we retrieve one single counterfactual while Baycon finds multiple by default.

\begin{table*}[]
\caption{Average run time (in seconds) across outliers for all datasets. Baseline, thanks to its simplicity, consistently outperforms other methods. DCFO generally outperforms Baycon and EACE thanks to its gradient-based optimisation.}
\centering
\begin{tabular}{lccccc}
\hline
Dataset & k & DCFO & BayCon & EACE & Baseline \\
\hline
MeanWhile1 & 20 & 1.6 $\pm$ 0.2 & 16 $\pm$ 2 & 2.52 $\pm$ 0.02 & \textbf{0.006 $\pm$ 0.002} \\
MindCave2 & 15 & 1.02 $\pm$ 0.04 & 9.7 $\pm$ 0.1 & 2.634 $\pm$ 0.009 & $\mathbf{2.3 \times 10^{-3}}$ $\pm$ $\mathbf{3 \times 10^{-5}}$ \\
SPECTF & 20 & 1.88 $\pm$ 0.06 & 17.9 $\pm$ 0.8 & 2.857 $\pm$ 0.003 & $\mathbf{5.9 \times 10^{-3}}$ $\pm$ $\mathbf{8 \times 10^{-4}}$ \\
ar3 & 10 & 2.69 $\pm$ NA & 6.50 $\pm$ NA & 2.01 $\pm$ NA & $\mathbf{1.4 \times 10^{-3}}$ $\pm$ $\mathbf{2 \times 10^{-19}}$ \\
ar5 & 15 & 7.47 $\pm$ NA & 26.09 $\pm$ NA & NA & $\mathbf{3.2 \times 10^{-3}}$ $\pm$ $\mathbf{3 \times 10^{-19}}$ \\
ar6 & 20 & 2.6 $\pm$ 0.2 & 11.8 $\pm$ 0.7 & 2.09 $\pm$ 0.01 & $\mathbf{2.3 \times 10^{-3}}$ $\pm$ $\mathbf{1 \times 10^{-4}}$ \\
autoPrice & 20 & 0.66 $\pm$ NA & 8.72 $\pm$ NA & 1.11 $\pm$ NA & $\mathbf{1.1 \times 10^{-3}}$ $\pm$ $\mathbf{1 \times 10^{-19}}$ \\
baskball & 10 & $8.6 \times 10^{-2}$ $\pm$ $6 \times 10^{-18}$ & 1.25 $\pm$ NA & 0.76 $\pm$ NA & $\mathbf{7 \times 10^{-4}}$ $\pm$ $\mathbf{5 \times 10^{-20}}$ \\
blood-transfusion-service-center & 20 & 0.99 $\pm$ 0.07 & 2.274 $\pm$ 0.007 & 1.36 $\pm$ 0.02 & $\mathbf{2.3 \times 10^{-3}}$ $\pm$ $\mathbf{2 \times 10^{-5}}$ \\
bodyfat & 10 & 0.172 $\pm$ 0.006 & 3.72 $\pm$ 0.06 & 1.123 $\pm$ 0.002 & $\mathbf{1.7 \times 10^{-3}}$ $\pm$ $\mathbf{3 \times 10^{-6}}$ \\
chscase census2 & 10 & 0.21 $\pm$ 0.01 & 2.02 $\pm$ 0.03 & 1.04 $\pm$ 0.01 & $\mathbf{2.4 \times 10^{-3}}$ $\pm$ $\mathbf{5 \times 10^{-6}}$ \\
chscase census6 & 10 & 0.229 $\pm$ 0.006 & 1.84 $\pm$ 0.03 & 1.06 $\pm$ 0.03 & $\mathbf{2.1 \times 10^{-3}}$ $\pm$ $\mathbf{2 \times 10^{-6}}$ \\
chscase vine1 & 20 & 1.66 $\pm$ NA & 4.49 $\pm$ NA & 0.76 $\pm$ NA & $\mathbf{6.4 \times 10^{-4}}$ $\pm$ $\mathbf{4 \times 10^{-20}}$ \\
confidence & 10 & 0.29 $\pm$ 0.03 & 1.25 $\pm$ 0.01 & 0.84 $\pm$ 0.04 & $\mathbf{6.5 \times 10^{-4}}$ $\pm$ $\mathbf{2 \times 10^{-6}}$ \\
diabetes & 15 & 0.51 $\pm$ 0.03 & 2.88 $\pm$ 0.04 & 1.19 $\pm$ 0.02 & $\mathbf{5 \times 10^{-3}}$ $\pm$ $\mathbf{1 \times 10^{-5}}$ \\
diabetes numeric & 20 & 0.50 $\pm$ NA & 1.47 $\pm$ NA & NA & \textbf{0.00 $\pm$ NA} \\
diggle table a1 & 20 & 9.9 $\pm$ 0.6 & \textbf{6.2 $\pm$ 0.3} & NA & NA \\
disclosure x noise & 20 & 0.36 $\pm$ 0.02 & 2.60 $\pm$ 0.07 & 0.82 $\pm$ 0.02 & $\mathbf{2.8 \times 10^{-3}}$ $\pm$ $\mathbf{8 \times 10^{-6}}$ \\
ecoli & 10 & 0.31 $\pm$ 0.02 & 2.09 $\pm$ 0.03 & 1.084 $\pm$ 0.010 & $\mathbf{1.8 \times 10^{-3}}$ $\pm$ $\mathbf{3 \times 10^{-6}}$ \\
glass & 15 & 0.60 $\pm$ 0.05 & 3.4 $\pm$ 0.1 & 1.11 $\pm$ 0.02 & $\mathbf{1.2 \times 10^{-3}}$ $\pm$ $\mathbf{2 \times 10^{-6}}$ \\
heart-statlog & 15 & 0.45 $\pm$ NA & 3.24 $\pm$ NA & 1.13 $\pm$ NA & $\mathbf{1.9 \times 10^{-3}}$ $\pm$ $\mathbf{2 \times 10^{-19}}$ \\
ionosphere & 10 & 4 $\pm$ 1 & 10 $\pm$ 1 & 2.37 $\pm$ 0.03 & $\mathbf{3.8 \times 10^{-3}}$ $\pm$ $\mathbf{2 \times 10^{-4}}$ \\
iris & 15 & $1.8 \times 10^{-1}$ $\pm$ $8 \times 10^{-18}$ & 1.77 $\pm$ NA & 0.91 $\pm$ NA & $\mathbf{8.5 \times 10^{-4}}$ $\pm$ $\mathbf{5 \times 10^{-20}}$ \\
kc1-top5 & 15 & 3.13 $\pm$ 0.09 & 29.5 $\pm$ 0.5 & 6.11 $\pm$ 0.02 & $\mathbf{3.9 \times 10^{-3}}$ $\pm$ $\mathbf{2 \times 10^{-4}}$ \\
kc2 & 10 & 1.8 $\pm$ 0.2 & 8.2 $\pm$ 0.8 & 1.680 $\pm$ 0.005 & $\mathbf{3.9 \times 10^{-3}}$ $\pm$ $\mathbf{1 \times 10^{-4}}$ \\
kc3 & 15 & 2.6 $\pm$ 0.1 & 18.3 $\pm$ 0.4 & 2.643 $\pm$ 0.002 & $\mathbf{3.9 \times 10^{-3}}$ $\pm$ $\mathbf{5 \times 10^{-5}}$ \\
libras move & 10 & 1.60 $\pm$ NA & 15.63 $\pm$ NA & 5.74 $\pm$ NA & $\mathbf{8.3 \times 10^{-3}}$ $\pm$ $\mathbf{1 \times 10^{-18}}$ \\
liver-disorders & 15 & 0.173 $\pm$ 0.004 & 2.29 $\pm$ 0.03 & 1.11 $\pm$ 0.02 & $\mathbf{1.8 \times 10^{-3}}$ $\pm$ $\mathbf{4 \times 10^{-6}}$ \\
longley & 10 & \textbf{3.2 $\pm$ 0.2} & 3.7 $\pm$ 0.2 & NA & NA \\
machine cpu & 10 & 0.29 $\pm$ 0.02 & 2.11 $\pm$ 0.10 & 1.01 $\pm$ 0.03 & $\mathbf{1 \times 10^{-3}}$ $\pm$ $\mathbf{2 \times 10^{-6}}$ \\
mfeat-zernike & 15 & 20 $\pm$ 3 & 24 $\pm$ 2 & 3.364 $\pm$ 0.003 & $\mathbf{1.6 \times 10^{-2}}$ $\pm$ $\mathbf{2 \times 10^{-4}}$ \\
mu284 & 20 & 1.12 $\pm$ NA & 3.16 $\pm$ NA & 0.99 $\pm$ NA & $\mathbf{1.6 \times 10^{-3}}$ $\pm$ $\mathbf{1 \times 10^{-19}}$ \\
no2 & 10 & 0.122 $\pm$ 0.004 & 2.61 $\pm$ 0.09 & 0.92 $\pm$ 0.01 & $\mathbf{3.1 \times 10^{-3}}$ $\pm$ $\mathbf{1 \times 10^{-5}}$ \\
pm10 & 15 & $1.2 \times 10^{-1}$ $\pm$ $1 \times 10^{-17}$ & 2.52 $\pm$ NA & 1.75 $\pm$ NA & \textbf{0.00 $\pm$ NA} \\
prnn fglass & 10 & 0.28 $\pm$ 0.03 & 2.14 $\pm$ 0.07 & 1.07 $\pm$ 0.05 & $\mathbf{1.1 \times 10^{-3}}$ $\pm$ $\mathbf{3 \times 10^{-6}}$ \\
pyrim & 20 & 1.24 $\pm$ 0.03 & 19 $\pm$ 1 & NA & $\mathbf{1.7 \times 10^{-3}}$ $\pm$ $\mathbf{3 \times 10^{-5}}$ \\
rabe 131 & 10 & $2 \times 10^{-1}$ $\pm$ $6 \times 10^{-18}$ & 3.75 $\pm$ NA & NA & $\mathbf{6.4 \times 10^{-4}}$ $\pm$ $\mathbf{7 \times 10^{-20}}$ \\
sleep & 15 & 0.414 $\pm$ 0.004 & 2.02 $\pm$ 0.04 & NA & $\mathbf{6.5 \times 10^{-4}}$ $\pm$ $\mathbf{1 \times 10^{-6}}$ \\
sonar & 15 & 1.58 $\pm$ NA & 18.31 $\pm$ NA & 3.73 $\pm$ NA & $\mathbf{2 \times 10^{-3}}$ $\pm$ $\mathbf{2 \times 10^{-19}}$ \\
steel-plates-fault & 15 & 6 $\pm$ 2 & 13 $\pm$ 1 & 2.61 $\pm$ 0.02 & $\mathbf{1.5 \times 10^{-2}}$ $\pm$ $\mathbf{4 \times 10^{-4}}$ \\
strikes & 20 & 0.94 $\pm$ 0.10 & 3.06 $\pm$ 0.04 & 1.44 $\pm$ 0.05 & $\mathbf{4.3 \times 10^{-3}}$ $\pm$ $\mathbf{3 \times 10^{-5}}$ \\
tecator & 10 & 1.83 $\pm$ 0.07 & 23 $\pm$ 1 & 8.072 $\pm$ 0.001 & $\mathbf{3.4 \times 10^{-3}}$ $\pm$ $\mathbf{2 \times 10^{-4}}$ \\
triazines & 15 & 1.77 $\pm$ NA & 34.53 $\pm$ NA & 3.88 $\pm$ NA & $\mathbf{3.5 \times 10^{-3}}$ $\pm$ $\mathbf{2 \times 10^{-19}}$ \\
vehicle & 15 & 1.6 $\pm$ 0.2 & 10.9 $\pm$ 0.6 & NA & $\mathbf{6.4 \times 10^{-3}}$ $\pm$ $\mathbf{8 \times 10^{-5}}$ \\
wdbc & 15 & 0.87 $\pm$ 0.03 & 10.5 $\pm$ 0.4 & 2.141 $\pm$ 0.002 & $\mathbf{5 \times 10^{-3}}$ $\pm$ $\mathbf{8 \times 10^{-5}}$ \\
wine & 10 & 0.50 $\pm$ 0.04 & 3.04 $\pm$ 0.09 & $1 \times 10^{0}$ $\pm$ $5 \times 10^{-4}$ & $\mathbf{1.1 \times 10^{-3}}$ $\pm$ $\mathbf{1 \times 10^{-6}}$ \\
wine-quality-red & 10 & 0.29 $\pm$ 0.02 & 3.08 $\pm$ 0.07 & 1.72 $\pm$ 0.04 & $\mathbf{1.7 \times 10^{-2}}$ $\pm$ $\mathbf{1 \times 10^{-5}}$ \\
wine-quality-white & 20 & 1.04 $\pm$ 0.04 & 5.19 $\pm$ 0.07 & 2.97 $\pm$ 0.04 & $\mathbf{1.3 \times 10^{-1}}$ $\pm$ $\mathbf{4 \times 10^{-5}}$ \\
wisconsin & 10 & 0.43 $\pm$ 0.05 & 17 $\pm$ 2 & 2.186 $\pm$ 0.003 & $\mathbf{2.6 \times 10^{-3}}$ $\pm$ $\mathbf{2 \times 10^{-4}}$ \\
\hline
\end{tabular}
\label{tab:runTime}
\end{table*}

\begin{figure}[tbp]
        \centering
        \includegraphics[width=1\linewidth]{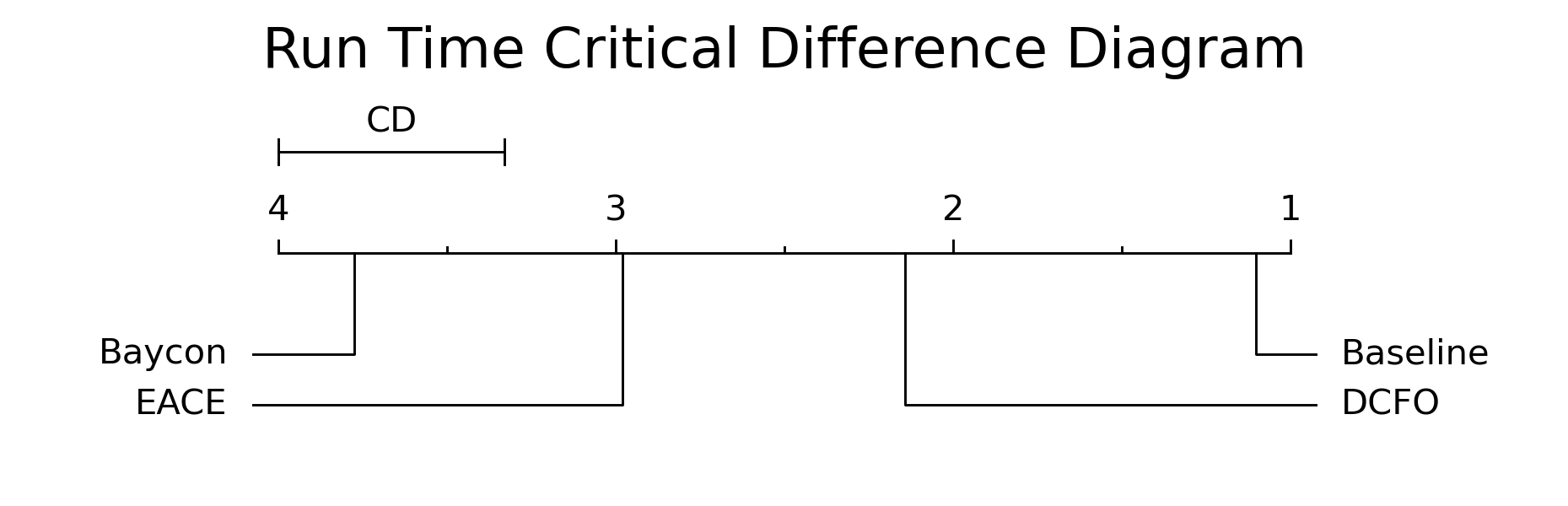}
        \caption{Run time Critical Difference diagram comparing run times across counterfactual generation methods. Baseline, the simplest algorithm, consistently ranks first. DCFO ranks higher than EACE and Baycon, with the result being statistically significant according to Nemenyi's test.}
        \label{fig:cdPlotRunTime} 
        \Description{}
    \end{figure}

Run times testify that Baseline is consistently the best method computationally. Baseline simply moves the outlier to the closest inlier: a simple and computationally cheap heuristics. Both Table~\ref{tab:runTime} and Figure~\ref{fig:cdPlotRunTime}, show that DCFO's gradient-based algorithm is more efficient than Bayesian and genetic algorithms, represented by Baycon and EACE respectively. 

\section{LOF's parameters sensitivity analysis}

We repeat here DCFO's comparison with studied baselines. Differently from Section~\ref{subsec:bestCounterfactual}, we expand the range of LOF's hyperparameters used, i.e. $k$ and the outliers' threshold. This sensitivity analysis whose results are reported in Table~\ref{tab:sensitivityProximity} and Table , shows that DCFO superior performance over baselines is robust to LOF's hyperparameters.

\begin{table*}[h!]
\caption{Proximity results for the sensitivity analysis over LOF's parameters, extended with the FullOpt baseline (without access to regions). Results are in line with Table~\ref{tab:fullProximity}, showing DCFO's robustness to various $k$ and threshold values, with DCFO having better proximity for every dataset over the methods also compared in Table~\ref{tab:fullProximity}. While FullOpt seems sometime better than DCFO, in reality it finds the same counterfactuals as DCFO but only in the "easy" and proximal cases in which the optimisation never leaves the starting region while not giving a solution in the other cases. Thus, results need to be read jointly with validity information which is extremely low for FullOpt.}
\centering                 
\begin{tabular}{lccccccc}
\hline
Dataset & k & Threshold & FullOpt & DCFO & BayCon & EACE & Baseline \\
\hline
MeanWhile1 & 8 & 1.31 & \textbf{0.01 $\pm$ \text{NA}} & 0.7 $\pm$ 0.2 & 1.1 $\pm$ 0.3 & 7 $\pm$ 2 & 3.5 $\pm$ 0.3 \\
MindCave2 & 5 & 1.54 & $\text{NA}$ & \textbf{0.5 $\pm$ 0.1} & 0.9 $\pm$ 0.2 & 10 $\pm$ 2 & 3.6 $\pm$ 0.4 \\
SPECTF & 20 & 1.39 & \textbf{0.4 $\pm$ 0.1} & 0.7 $\pm$ 0.2 & 1.9 $\pm$ 0.4 & 4.0 $\pm$ 0.2 & 6.60 $\pm$ 0.01 \\
ar3 & 10 & 1.75 & $\text{NA}$ & \textbf{0.5 $\pm$ 0.3} & 0.8 $\pm$ 0.2 & 5.3 $\pm$ 0.3 & 3.3 $\pm$ 0.3 \\
ar5 & 8 & 1.38 & $\text{NA}$ & \textbf{3 $\pm$ 2} & 5 $\pm$ 3 & NA & 6 $\pm$ 2 \\
ar6 & 20 & 1.42 & \textbf{0.09 $\pm$ \text{NA}} & 0.41 $\pm$ 0.06 & 0.9 $\pm$ 0.3 & 10 $\pm$ 2 & 3.6 $\pm$ 0.8 \\
autoPrice & 16 & 1.32 & $\text{NA}$ & \textbf{0.22 $\pm$ 0.08} & 0.4 $\pm$ 0.1 & 3.1 $\pm$ 0.4 & 2.7 $\pm$ 0.3 \\
baskball & 8 & 1.19 & $\text{NA}$ & \textbf{0.03 $\pm$ NA} & 0.07 $\pm$ NA & 0.45 $\pm$ NA & 0.87 $\pm$ NA \\
blood-transfusion-service-center & 12 & 1.15 & $\text{NA}$ & \textbf{0.3 $\pm$ 0.1} & 0.3 $\pm$ 0.1 & 3 $\pm$ 1 & 0.5 $\pm$ 0.1 \\
bodyfat & 14 & 1.29 & $\text{NA}$ & \textbf{0.22 $\pm$ 0.09} & 0.4 $\pm$ 0.1 & 3 $\pm$ 1 & 2.33 $\pm$ 0.07 \\
chscase census2 & 8 & 1.30 & \textbf{0.12 $\pm$ 0.06} & \textbf{0.12 $\pm$ 0.02} & 0.20 $\pm$ 0.03 & 1.4 $\pm$ 0.2 & 1.81 $\pm$ 0.07 \\
chscase census6 & 10 & 1.29 & $\text{NA}$ & \textbf{0.18 $\pm$ 0.05} & 0.25 $\pm$ 0.06 & 1.2 $\pm$ 0.1 & 1.06 $\pm$ 0.09 \\
chscase vine1 & 7 & 1.21 & $\text{NA}$ & \textbf{0.09 $\pm$ NA} & 0.13 $\pm$ NA & 1.93 $\pm$ NA & 1.75 $\pm$ NA \\
confidence & 5 & 1.56 & $\text{NA}$ & \textbf{0.13 $\pm$ 0.06} & 0.16 $\pm$ 0.06 & 2.2 $\pm$ 0.4 & 0.24 $\pm$ 0.09 \\
diabetes & 9 & 1.24 & $\text{NA}$ & \textbf{0.15 $\pm$ 0.02} & 0.23 $\pm$ 0.03 & 1.8 $\pm$ 0.1 & 1.59 $\pm$ 0.05 \\
diabetes numeric & 12 & 1.42 & $\text{NA}$ & \textbf{0.4 $\pm$ 0.1} & 0.6 $\pm$ 0.3 & NA & 1.2 $\pm$ 0.3 \\
diggle table a1 & 7 & 1.39 & $\text{NA}$ & \textbf{0.17 $\pm$ NA} & 0.38 $\pm$ NA & 0.56 $\pm$ NA & 1.43 $\pm$ NA \\
disclosure x noise & 11 & 1.30 & \textbf{0.08 $\pm$ 0.02} & 0.15 $\pm$ 0.04 & 0.21 $\pm$ 0.04 & 1.4 $\pm$ 0.4 & 0.56 $\pm$ 0.06 \\
ecoli & 13 & 1.34 & \textbf{0.16 $\pm$ 0.2} & \textbf{0.16 $\pm$ 0.03} & 0.23 $\pm$ 0.04 & 1.4 $\pm$ 0.3 & 0.87 $\pm$ 0.06 \\
glass & 13 & 2.30 & \textbf{0.3 $\pm$ 0.1} & 2.1 $\pm$ 0.2 & 3.67 $\pm$ 0.09 & 13 $\pm$ 4 & 3.5 $\pm$ 0.4 \\
heart-statlog & 17 & 1.18 & \textbf{0.08 $\pm$ \text{NA}} & 0.17 $\pm$ 0.03 & 0.28 $\pm$ 0.05 & 2.2 $\pm$ 0.2 & 3.0 $\pm$ 0.1 \\
ionosphere & 13 & 3.15 & \textbf{1.11 $\pm$ \text{NA}} & 2.6 $\pm$ 0.4 & NA & 13 $\pm$ 1 & 7.4 $\pm$ 0.4 \\
iris & 20 & 1.24 & $0.15 \pm \text{NA}$ & \textbf{0.031 $\pm$ 0.008} & 0.057 $\pm$ 0.009 & 0.42 $\pm$ 0.05 & 0.46 $\pm$ 0.06 \\
kc1-top5 & 15 & 1.89 & \textbf{0.01 $\pm$ \text{NA}} & 1.1 $\pm$ 0.7 & 7.7 $\pm$ 0.5 & 23 $\pm$ 8 & 7 $\pm$ 3 \\
kc2 & 6 & 1.53 & $0.3 \pm 0.2$ & \textbf{0.18 $\pm$ 0.04} & 0.6 $\pm$ 0.3 & 2.8 $\pm$ 0.4 & 0.8 $\pm$ 0.1 \\
kc3 & 19 & 1.87 & \textbf{0.14 $\pm$ 0.03} & 0.7 $\pm$ 0.3 & 2.0 $\pm$ 0.3 & 15 $\pm$ 3 & 2.4 $\pm$ 0.6 \\
libras move & 12 & 1.19 & $\text{NA}$ & \textbf{0.01 $\pm$ NA} & 0.05 $\pm$ NA & 1.78 $\pm$ NA & 3.08 $\pm$ NA \\
liver-disorders & 6 & 1.30 & $0.23 \pm \text{NA}$ & \textbf{0.18 $\pm$ 0.04} & 0.24 $\pm$ 0.05 & 1.2 $\pm$ 0.1 & 0.94 $\pm$ 0.08 \\
longley & 11 & 1.06 & $\text{NA}$ & \textbf{0.2 $\pm$ 0.2} & 0.3 $\pm$ 0.2 & NA & 1.5 $\pm$ 0.1 \\
machine cpu & 12 & 1.56 & $0.5 \pm 0.1$ & \textbf{0.26 $\pm$ 0.06} & 0.39 $\pm$ 0.08 & 2.6 $\pm$ 0.8 & 0.9 $\pm$ 0.1 \\
mfeat-zernike & 12 & 1.17 & $\text{NA}$ & \textbf{0.20 $\pm$ 0.05} & 0.5 $\pm$ 0.1 & 3.8 $\pm$ 0.4 & 4.1 $\pm$ 0.2 \\
mu284 & 19 & 1.30 & $\text{NA}$ & \textbf{0.16 $\pm$ 0.05} & 0.3 $\pm$ 0.1 & 2.0 $\pm$ 0.3 & 1.1 $\pm$ 0.1 \\
no2 & 13 & 1.21 & $\text{NA}$ & \textbf{0.11 $\pm$ 0.03} & 0.18 $\pm$ 0.04 & 1.2 $\pm$ 0.1 & 1.23 $\pm$ 0.09 \\
pm10 & 8 & 1.20 & $\text{NA}$ & \textbf{0.14 $\pm$ 0.03} & 0.20 $\pm$ 0.04 & 1.3 $\pm$ 0.1 & 1.09 $\pm$ 0.07 \\
prnn fglass & 8 & 2.45 & $\text{NA}$ & \textbf{0.6 $\pm$ 0.5} & 2.6 $\pm$ 0.2 & 5 $\pm$ 2 & 3.3 $\pm$ 0.4 \\
pyrim & 6 & 1.33 & $\text{NA}$ & \textbf{0.69 $\pm$ NA} & 1.20 $\pm$ NA & 4.86 $\pm$ NA & 2.88 $\pm$ NA \\
rabe 131 & 10 & 1.10 & $\text{NA}$ & \textbf{0.08 $\pm$ NA} & 0.11 $\pm$ NA & 1.25 $\pm$ NA & 1.08 $\pm$ NA \\
sleep & 14 & 1.35 & $\text{NA}$ & \textbf{2 $\pm$ 1} & 3 $\pm$ 1 & NA & 3 $\pm$ 1 \\
sonar & 18 & 1.26 & $\text{NA}$ & \textbf{0.32 $\pm$ 0.09} & 0.9 $\pm$ 0.2 & 9.7 $\pm$ 0.9 & 7.2 $\pm$ 0.4 \\
steel-plates-fault & 18 & 1.21 & \textbf{0.4 $\pm$ 0.2} & 0.41 $\pm$ 0.04 & 0.65 $\pm$ 0.06 & 5.7 $\pm$ 0.3 & 2.7 $\pm$ 0.1 \\
strikes & 13 & 1.25 & $\text{NA}$ & \textbf{0.24 $\pm$ 0.04} & 0.33 $\pm$ 0.07 & 1.4 $\pm$ 0.1 & 0.9 $\pm$ 0.1 \\
tecator & 14 & 1.46 & $\text{NA}$ & \textbf{0.06 $\pm$ 0.03} & 0.12 $\pm$ 0.04 & 1.2 $\pm$ 0.1 & 5.1 $\pm$ 0.7 \\
triazines & 9 & 1.73 & $\text{NA}$ & \textbf{2 $\pm$ 1} & 8 $\pm$ 2 & 17 $\pm$ 6 & 6 $\pm$ 2 \\
vehicle & 8 & 1.16 & \textbf{0.01 $\pm$ \text{NA}} & 0.17 $\pm$ 0.03 & 0.27 $\pm$ 0.04 & 4.1 $\pm$ 0.3 & 1.59 $\pm$ 0.07 \\
wdbc & 5 & 1.30 & $0.73 \pm \text{NA}$ & \textbf{0.30 $\pm$ 0.05} & 0.6 $\pm$ 0.1 & 5.0 $\pm$ 0.5 & 3.6 $\pm$ 0.2 \\
wine & 17 & 1.16 & $0.55 \pm \text{NA}$ & \textbf{0.26 $\pm$ 0.07} & 0.42 $\pm$ 0.09 & 2.1 $\pm$ 0.8 & 2.3 $\pm$ 0.1 \\
wine-quality-red & 8 & 1.33 & \textbf{0.2 $\pm$ 0.1} & 0.23 $\pm$ 0.04 & 0.35 $\pm$ 0.05 & 2.9 $\pm$ 0.1 & 1.90 $\pm$ 0.08 \\
wine-quality-white & 10 & 1.28 & $0.40 \pm \text{NA}$ & \textbf{0.21 $\pm$ 0.01} & 0.33 $\pm$ 0.02 & 2.79 $\pm$ 0.08 & 1.61 $\pm$ 0.03 \\
wisconsin & 19 & 1.32 & \textbf{0.12 $\pm$ \text{NA}} & 0.4 $\pm$ 0.1 & 0.8 $\pm$ 0.3 & 7 $\pm$ 2 & 4.8 $\pm$ 0.3 \\
\hline
\end{tabular}
\label{tab:sensitivityProximity}
\end{table*}

\begin{table*}[h!]
\caption{Validity results for the sensitivity analysis over LOF's parameters, extended with the FullOpt validity results. DCFO demonstrates high robustness across both threshold sweeps and full optimisation.}
\centering                 
\begin{tabular}{l c c c c c c c}
\hline
Dataset & k & Threshold & FullOpt & DCFO & BayCon & EACE & Baseline \\
\hline
MeanWhile1 & 8 & 1.31 & 0.06 & \textbf{1.00} & 0.72 & 0.41 & 0.93 \\
MindCave2 & 5 & 1.54 & 0    & \textbf{1.00} & 0.89 & 0.78 & \textbf{1.00} \\
SPECTF & 20 & 1.39 & 0.00 & \textbf{1.00} & 0.66 & 0.13 & 0.84 \\
ar3 & 10 & 1.75 & 0    & \textbf{1.00} & 0.67 & 0.33 & 0.50 \\
ar5 & 8 & 1.38 & 0    & \textbf{1.00} & \textbf{1.00} & 0.00 & \textbf{1.00} \\
ar6 & 20 & 1.42 & 0.20 & \textbf{1.00} & 0.80 & 0.50 & 0.90 \\
autoPrice & 16 & 1.32 & 0    & \textbf{1.00} & 0.96 & 0.26 & 0.74 \\
baskball & 8 & 1.19 & 0    & \textbf{1.00} & \textbf{1.00} & 0.10 & 0.90 \\
blood-transfusion-service-center & 12 & 1.15 & 0    & \textbf{1.00} & 0.87 & 0.35 & \textbf{1.00} \\
bodyfat & 14 & 1.29 & 0    & \textbf{1.00} & 0.76 & 0.24 & \textbf{1.00} \\
chscase census2 & 8 & 1.30 & 0.12 & \textbf{1.00} & \textbf{1.00} & 0.46 & 0.88 \\
chscase census6 & 10 & 1.29 & 0    & \textbf{1.00} & 0.92 & 0.29 & 0.94 \\
chscase vine1 & 7 & 1.21 & -- & \textbf{1.00} & 0.80 & 0.20 & \textbf{1.00} \\
confidence & 5 & 1.56 & 0    & \textbf{1.00} & \textbf{1.00} & \textbf{1.00} & \textbf{1.00} \\
diabetes & 9 & 1.24 & 0    & \textbf{1.00} & 0.96 & 0.54 & 0.91 \\
diabetes numeric & 12 & 1.42 & 0    & \textbf{1.00} & \textbf{1.00} & 0.00 & \textbf{1.00} \\
diggle table a1 & 7 & 1.39 & 0    & \textbf{1.00} & \textbf{1.00} & 0.20 & \textbf{1.00} \\
disclosure x noise & 11 & 1.30 & 0.12 & \textbf{1.00} & 0.98 & 0.18 & 0.91 \\
ecoli & 13 & 1.34 & 0.08 & \textbf{1.00} & 0.84 & 0.26 & \textbf{1.00} \\
glass & 13 & 2.30 & 0.06 & \textbf{1.00} & 0.46 & 0.54 & 0.92 \\
heart-statlog & 17 & 1.18 & 0.50 & \textbf{1.00} & \textbf{1.00} & 0.50 & 0.89 \\
ionosphere & 13 & 3.15 & 0.01 & \textbf{1.00} & 0.00 & 0.97 & 0.97 \\
iris & 20 & 1.24 & 0.17 & \textbf{1.00} & \textbf{1.00} & 0.35 & \textbf{1.00} \\
kc1-top5 & 15 & 1.89 & 0.05 & \textbf{1.00} & 0.88 & 0.38 & \textbf{1.00} \\
kc2 & 6 & 1.53 & 0.06 & \textbf{1.00} & 0.72 & 0.45 & 0.91 \\
kc3 & 19 & 1.87 & 0.08 & \textbf{1.00} & 0.67 & 0.50 & 0.94 \\
libras move & 12 & 1.19 & 0    & \textbf{1.00} & 0.49 & 0.02 & 0.88 \\
liver-disorders & 6 & 1.30 & 0.05 & \textbf{1.00} & 0.98 & 0.46 & 0.91 \\
longley & 11 & 1.06 & -- & \textbf{1.00} & \textbf{1.00} & 0.00 & \textbf{1.00} \\
machine cpu & 12 & 1.56 & 0.06 & \textbf{1.00} & \textbf{1.00} & 0.23 & 0.88 \\
mfeat-zernike & 12 & 1.17 & 0    & \textbf{1.00} & 0.71 & 0.31 & 0.97 \\
mu284 & 19 & 1.30 & 0    & \textbf{1.00} & 0.82 & 0.43 & 0.86 \\
no2 & 13 & 1.21 & 0    & \textbf{1.00} & 0.98 & 0.46 & 0.98 \\
pm10 & 8 & 1.20 & 0    & \textbf{1.00} & 0.96 & 0.45 & 0.89 \\
prnn fglass & 8 & 2.45 & 0    & \textbf{1.00} & 0.55 & 0.55 & \textbf{1.00} \\
pyrim & 6 & 1.33 & 0    & \textbf{1.00} & 0.91 & 0.09 & \textbf{1.00} \\
rabe 131 & 10 & 1.10 & 0    & \textbf{1.00} & \textbf{1.00} & 0.25 & 0.75 \\
sleep & 14 & 1.35 & 0    & \textbf{1.00} & 0.75 & 0.00 & \textbf{1.00} \\
sonar & 18 & 1.26 & 0    & \textbf{1.00} & 0.96 & 0.39 & 0.93 \\
steel-plates-fault & 18 & 1.21 & 0.07 & \textbf{1.00} & 0.54 & 0.46 & 0.87 \\
strikes & 13 & 1.25 & 0    & \textbf{1.00} & 0.86 & 0.43 & 0.98 \\
tecator & 14 & 1.46 & 0    & \textbf{1.00} & 0.82 & 0.14 & 0.95 \\
triazines & 9 & 1.73 & 0    & \textbf{1.00} & \textbf{1.00} & 0.20 & 0.85 \\
vehicle & 8 & 1.16 & 0.20 & \textbf{1.00} & 0.65 & 0.62 & 0.94 \\
wdbc & 5 & 1.30 & 0.04 & \textbf{1.00} & 0.82 & 0.49 & 0.93 \\
wine & 17 & 1.16 & 0.08 & \textbf{1.00} & 0.75 & 0.29 & 0.79 \\
wine-quality-red & 8 & 1.33 & 0.08 & \textbf{1.00} & 0.85 & 0.65 & 0.97 \\
wine-quality-white & 10 & 1.28 & 0.01 & \textbf{1.00} & 0.84 & 0.67 & 0.95 \\
wisconsin & 19 & 1.32 & 0.10 & \textbf{1.00} & \textbf{1.00} & 0.30 & 0.85 \\
\hline
\end{tabular}
\label{tab:sensitivityValidity}
\end{table*}

\section{Space partition ablation study}

In order to properly analyse the benefits of the space partition, we run our gradient-based constrained optimisation without giving the algorithm access to regions, we call this baseline FullOpt. As the non-continuity of the LOF function would suggest, FullOpt fails to retrieve counterfactuals when the algorithm leaves the starting region. When instead the counterfactual is close enough to the original instance to be in the same region, FullOpt works with a fully differentiable function and finds the same counterfactual as DCFO. Table~\ref{tab:sensitivityValidity} and Table~\ref{tab:sensitivityProximity} show FullOpt's proximity and validity results. The tables highlight how FullOpt fails to reliably produce counterfactual explanations as, without access to regions, LOF's being non-continuous prevents the gradient-based optimiser to find a valid solution. In table~\ref{tab:sensitivityProximity}, FullOpt seems to show better proximity than DCFO for some datasets. In reality, when FullOpt does find a counterfactual, it finds the same one as DCFO but only on the "easy" case in which the optimisation does not leave the initial region i.e. the algorithm exploits only an area where LOF is differentiable. For all the more challenging cases in which the optimisation leaves the starting region, FullOpt fails to find a counterfactual. Thus, the proximity results need to be read jointly with the validity numbers in Table~\ref{tab:sensitivityValidity}.

\end{document}